\documentclass[11pt]{article} % use larger type; default would be 10pt

\usepackage[utf8]{inputenc} % set input encoding (not needed with XeLaTeX)

\usepackage[letterpaper]{geometry}

\usepackage{graphicx} % support the \includegraphics command and options
\usepackage{booktabs} % for much better looking tables
\usepackage{array} % for better arrays (eg matrices) in maths
\usepackage{paralist} % very flexible & customisable lists (eg. enumerate/itemize, etc.)
\usepackage{verbatim} % adds environment for commenting out blocks of text & for better verbatim
\usepackage{mathrsfs}
\usepackage{amssymb}
\usepackage{amsthm}
\usepackage{amsmath,amsfonts,amssymb,esint}
\usepackage{graphics}
\usepackage{enumerate}
\usepackage{mathtools}
\usepackage{xfrac}
\usepackage{bbm}
\usepackage{subcaption}
\usepackage{times}

\usepackage[usenames,dvipsnames]{xcolor}
\usepackage[colorlinks=true, pdfstartview=FitV, linkcolor=blue, citecolor=blue, urlcolor=blue]{hyperref}

\numberwithin{equation}{section}

\newtheorem{theorem}{Theorem}[section]

\newtheorem{corollary}[theorem]{Corollary}
\newtheorem{proposition}[theorem]{Proposition}
\newtheorem{lemma}[theorem]{Lemma}

\theoremstyle{definition}
\newtheorem{remark}[theorem]{Remark}

\newtheorem{assumption}[theorem]{Assumption}

\renewcommand*{\tilde}{\widetilde}

\everymath{\displaystyle}
\DeclareMathOperator{\argmin}{argmin}

\usepackage[nottoc,notlot,notlof]{tocbibind}
\allowdisplaybreaks

\begin{document}
%\maketitle
\begin{center}
	\Large \bf Non-Reversible Langevin Algorithms for Constrained Sampling
\end{center}

\author{}
\begin{center}
	{Hengrong Du}\,\footnote{Department of Mathematics, University of California, Irvine,
410P Rowland Hall, Irvine, CA 92697-3875, United States of America;
		hengrond@uci.edu},
	{Qi Feng}\,\footnote{Department of Mathematics, Florida State University, 1017 Academic Way, Tallahassee, FL-32306, United States of America;
		qfeng2@fsu.edu},
        		  {Changwei Tu}\,\footnote{Hong Kong University of Science and Technology (Guangzhou), Guangzhou, Guangdong Province, People's Republic of China; ctu570@connect.hkust-gz.edu.cn
  },
		  {Xiaoyu Wang}\,\footnote{Hong Kong University of Science and Technology (Guangzhou), Guangzhou, Guangdong Province, People's Republic of China;
  xiaoyuwang@hkust-gz.edu.cn},
	Lingjiong Zhu\,\footnote{Department of Mathematics, Florida State University, 1017 Academic Way, Tallahassee, FL-32306, United States of America; zhu@math.fsu.edu
	}
\end{center}
\begin{center}
	\today
\end{center}

\begin{abstract}
We consider the constrained sampling problem where the goal is to sample from a target distribution on a constrained domain. 
We propose skew-reflected non-reversible Langevin dynamics (SRNLD), a continuous-time
stochastic differential equation with skew-reflected boundary. 
We obtain a non-asymptotic convergence rate of SRNLD to the target distribution in both total variation and $1$-Wasserstein distances.
By breaking reversibility, we show that the convergence is faster than
the special case of the reversible dynamics.
Based on the discretization of SRNLD, we propose skew-reflected non-reversible Langevin Monte Carlo (SRNLMC), and obtain non-asymptotic 
discretization error from SRNLD, and convergence guarantees
to the target distribution in $1$-Wasserstein distance.
We show better performance guarantees than the projected Langevin Monte Carlo in the literature that is based on the reversible dynamics.
Numerical experiments are provided for both synthetic and real datasets to show efficiency of the proposed algorithms.
\end{abstract}

\section{Introduction}\label{sec:intro}

We consider the problem of sampling a distribution $\pi$
on a convex constrained domain $\mathcal{C}\subsetneq\mathbb{R}^{d}$
with probability density function 
\begin{equation} 
\pi(x)\propto\exp(-f(x)),\ x\in \mathcal{C},
    \label{eq-target}
\end{equation}
for a function $f:\mathbb{R}^d \to \mathbb{R}$. The sampling problem 
for both constrained domain $\mathcal{C}\subsetneq\mathbb{R}^{d}$ and unconstrained domain $\mathcal{C}=\mathbb{R}^{d}$ is a fundamental problem that arises in many applications, including Bayesian statistical inference \cite{gelman1995bayesian}, Bayesian formulations of inverse problems \cite{stuart2010inverse}, as well as Bayesian classification and regression tasks in machine learning \cite{andrieu2003introduction,teh2016consistency,DistMCMC19,GHZ2022,GIWZ2024}.

In the literature, \cite{bubeck2015finite, bubeck2018sampling} studied the projected Langevin Monte Carlo (PLMC) algorithm for constrained sampling that projects the iterates back to the constraint set after applying the Langevin step:
\begin{equation}\label{projected:Langevin}
x_{k+1}=\mathcal{P}_{\mathcal{C}}\left(x_{k}-\eta\nabla f(x_{k})+\sqrt{2\eta}\xi_{k+1}\right),    
\end{equation}
where $\mathcal{P}_{\mathcal{C}}$ is the projection onto the 
set $\mathcal{C}$, $\eta>0$ is the stepsize, $\xi_{k}$ are i.i.d. Gaussian random vectors $\mathcal{N}(0,I)$ and the dynamics \eqref{projected:Langevin} is based on the continuous-time overdamped Langevin stochastic differential equation (SDE) with reflected boundary: 
\begin{equation}
dX_{t}=-\nabla f(X_{t})dt+\sqrt{2}dW_{t}+\nu(X_{t})L(dt),
\end{equation}
where the term $\nu(X_{t})L(dt)$ ensures that $X_{t}\in\mathcal{C}$ 
for every $t$ given that $X_{0}\in\mathcal{C}$. 
In particular, $\int_{0}^{t}\nu(X_{s})L(ds)$ is a bounded variation reflection process
and the measure $L(dt)$ is such that $L([0,t])$ is finite, 
$L(dt)$ is supported on $\{t|X_{t}\in\partial\mathcal{C}\}$.

It is shown in \cite{bubeck2018sampling} that $\tilde{\mathcal{O}}(d^{12}/\varepsilon^{12})$ iterations are sufficient for having $\varepsilon$-error in the total variation (TV) distance with respect to the target distribution when the gradients are exact where the notation $\tilde{\mathcal{O}}(\cdot)$ hides some logarithmic factors. 
%%%%%%%%%%%%%%%%%%%%%%%%%%%%%%%%%%%%%%%%
\cite{Lamperski2021} considers the projected stochastic gradient Langevin dynamics in the setting of non-convex smooth Lipschitz $f$ on a convex body where the gradient noise is assumed to have finite variance with a uniform sub-Gaussian structure. The author shows that $\tilde{\mathcal{O}}\left(d^4/\varepsilon^{4}\right)$ iterations suffice in the $1$-Wasserstein metric; see also \cite{zheng2022constrained}.
%%%%%%%%%%%%%%%%%%%%%%%%%%%%%%%%%%%%%%%%
In addition, proximal Langevin Monte Carlo
is proposed in \cite{Brosse} for constrained sampling and a complexity of $\tilde{\mathcal{O}}\left(d^5/\varepsilon^6\right)$ is obtained. \cite{SR2020} further studies the proximal stochastic gradient Langevin algorithm
from a primal-dual perspective. 
%%%%%%%%%%%%%%%%%%%%%%%%%%%%%%%%%%%%%%%%%%%%%%%%
Mirror descent-based Langevin algorithms (see e.g. \cite{hsieh2018mirrored,Chewi2020,Zhang2020,TaoMirror2021,Ahn2021})
can also be used for constrained sampling. Mirrored Langevin dynamics
was proposed in \cite{hsieh2018mirrored}, inspired by the classical mirror descent in optimization. 
%%%%%%%%%%%%%%%%%%%%%%%%%%%%%
Very recently, inspired by the penalty method in the optimization literature, 
penalized Langevin Monte Carlo algorithms are proposed and studied in \cite{GHZ2022},
where the objective $f$ can be non-convex in general, and the better dependency
on the dimension $d$ is achieved in the complexity compared to the previous literature. Furthermore, constrained non-convex exploration combined with replica-exchange Langevin dynamics has been proposed and studied in \cite{constraint_replica}, which shows that reducing the diameter of the domain enhances mixing rates. 

In the literature, non-reversible Langevin SDE and the associated algorithms (supported on $\mathbb{R}^{d}$) have been studied
and are known to converge to the Gibbs distribution faster than the overdamped Langevin SDEs and algorithms. 
This motivates us to investigate whether non-reversibility can help in the context of projected Langevin
for constrained sampling.
In particular, by adding a (state-dependent) \textit{anti-symmetric} matrix $J=-J^{\top}$ 
to the overdamped Langevin diffusion, 
a non-reversible Langevin SDE takes the form:
\begin{equation}\label{eqn:anti}
dX_{t}=-(I+J(X_{t}))\nabla f(X_{t})dt+\sqrt{2}dW_{t},
\end{equation}
where $W_{t}$ is a standard $d$-dimensional Brownian motion, 
then the stationary distribution has the density $\pi(x)\propto
e^{-f(x)}$
which is the same as the Gibbs distribution of the overdamped Langevin. This dynamics is called \emph{non-reversible Langevin dynamics} (NLD) because this diffusion is a non-reversible Markov process due to the addition of the $J$ matrix whereas the overdamped Langevin diffusion is reversible (see \cite{HHS93,HHS05} for
details) with optimal choice of $J$ discussed in \cite{Lelievre-optdrift,WHC2014}. The main idea explored here is that non-reversible processes converge to their equilibrium often faster than their reversible counterparts \cite{HHS93,HHS05}
and this has been applied
to sampling \cite{rey2016improving,DLP2016, reyGraphs,DPZ17,reyLDP,FSS20}
and non-convex optimization \cite{GGZ2,HWGGZ20}.

In this paper, we introduce the \emph{skew-reflected non-reversible Langevin dynamics} (SRNLD):
\begin{equation}\label{eqn:anti:reflected}
dX_{t}=-(I+J(X_{t}))\nabla f(X_{t})dt+\sqrt{2}dW_{t}+\nu^J(X_{t})L(dt),
\end{equation}
where for every $x$, $J(x)$ is an anti-symmetric matrix, i.e. $J(x)=-(J(x))^{\top}$
and $\Vert J\Vert_{\infty}:=\sup_{x\in\mathcal{C}}\Vert J(x)\Vert<\infty$.
The term $\nu^J(X_{t})L(dt)$ ensures that $X_{t}\in\mathcal{C}$ 
for every $t$ given that $X_{0}\in\mathcal{C}$. 
In particular, $\int_{0}^{t}\nu^J(X_{s})L(ds)$ is a bounded variation \emph{skew-reflection} process
and the measure $L(dt)$ is such that $L([0,t])$ is finite, 
$L(dt)$ is supported on $\{t|X_{t}\in\partial\mathcal{C}\}$, where the skew-reflection is defined thorough the following \emph{skew unit normal vector}: 
\begin{equation}\label{skew unit norm vector}
    \nu^J(X_{t}) = \frac{(I+J(X_{t}))\nu(X_{t})}{\sqrt{\|\nu(X_{t})\|^2+\|J(X_{t})\nu(X_{t})\|^2}},
\end{equation}
where $\nu(X_{t})\in\mathcal{N}_{\mathcal{C}}(X_{t})$ is referring to the unit inner normal vector, and we denote  $\mathcal{N}_{\mathcal{C}}(x)$ as the normal cone of $\mathcal{C}$ at $x$. We assume the \text{skew-matrix} $J$ satisfies the condition that $\langle \nu^J, \nu\rangle \ge \delta_0> 0$ for some positive constant $\delta_0$.
Under these conditions, the reflection process is uniquely defined. The
skew-reflection $\nu^J$ introduced here in \eqref{skew unit norm vector} is
different from \cite{bubeck2018sampling} due to the non-reversible dynamics
\eqref{eqn:anti}. In the literature, stochastic differential equations with
reflection and oblique reflection have been studied in
\cite{tanaka1979stochastic, LionsSznitman, costantini1992skorohod,bossy2004symmetrized}. However, to
the best of our knowledge, the invariant distribution of oblique reflection
processes and their exponential convergence to the invariant distribution have
not been studied. The newly proposed reflection is compatible with the boundary
condition and keeps the invariant measure of \eqref{eqn:anti} unchanged. To be
precise, the skew-reflection is realized through the following procedures.
Define the projection $\mathcal P_{\mathcal C}$ and reflection $\mathcal
R_{\mathcal C}$ as below,
\begin{equation}
    \mathcal P_{\mathcal C}(x):=\argmin_{y\in\bar{\mathcal C}} \| y-x\|,\quad \mathcal R_{\mathcal C}(x):=2\mathcal P_{\mathcal C}(x)-x.
\end{equation}
We then define the following skew-reflection:
\begin{equation}\label{defn: RCJ}
    \mathcal R^J_{\mathcal C}(x):=(I+J(x)) (\mathcal P_{\mathcal C}(x)-x)+\mathcal P_{\mathcal C}(x).
\end{equation}
Motivated by such a skew-reflection, we propose the following \emph{skew-projection}:
\begin{equation}\label{defn: PCJ}
    \mathcal P^J_{\mathcal C}(x):=\argmin_{y\in\bar{\mathcal C}} \left\langle y-x, \nu^J(\mathcal P_{\mathcal C}(x))\right\rangle.
\end{equation}
Observe that we have $x- \mathcal P^J_{\mathcal C}(x)$ parallel to $\mathcal R^J_{\mathcal C}(x)- \mathcal P_{\mathcal C}(x)$.
To implement SRNLD \eqref{eqn:anti:reflected}, in practice, 
we propose the \emph{skew-reflected non-reversible Langevin Monte Carlo} (SRNLMC) algorithm:
%\begin{equation}\label{eqn:algorithm}
%x_{k+1}=\mathcal{R}^J_{\mathcal{C}}\left(x_{k}-\eta(I+J(x_{k}))\nabla f(x_{k})+\sqrt{2\eta}\xi_{k+1}\right),
%\end{equation}
\begin{equation}\label{eqn:algorithm}
x_{k+1}=\mathcal{P}^J_{\mathcal{C}}\left(x_{k}-\eta(I+J(x_{k}))\nabla f(x_{k})+\sqrt{2\eta}\xi_{k+1}\right),
\end{equation}
% where $\mathcal{R}^J_{\mathcal{C}}$ is the skew-reflection onto $\mathcal{C}$: 
% \begin{equation}
%     \mathcal R^J_{\mathcal C}(x):=(I+J(x)) (\mathcal P_{\mathcal C}(x)-x)+\mathcal P_{\mathcal C}(x),
% \end{equation}
% where the projection $\mathcal P_{\mathcal C}$ is defined as:
% \begin{equation}
%     \mathcal P_{\mathcal C}(x):=\argmin_{y\in\bar{\mathcal C}} \| y-x\|,
% \end{equation}
where $\xi_{k}$ are i.i.d. Gaussian random vectors $\mathcal{N}(0,I)$. 
In particular, when $J(\cdot)\equiv 0$, 
the algorithm SRNLMC \eqref{eqn:algorithm} reduces
to the projected Langevin Monte Carlo algorithm (PLMC) in the literature \cite{bubeck2015finite, bubeck2018sampling}; When $J(\cdot)\neq 0$, the skew-reflected algorithm is equivalent to the skew projection onto the boundary $\partial\mathcal C$ parallel to $\nu^J$; see \cite{bossy2004symmetrized}[Proposition 1].

In this paper, we are interested in studying 
the continuous-time SRNLD \eqref{eqn:anti:reflected}
including the non-asymptotic convergence performance, 
as well as the discretization error of
the discrete-time algorithm, SRNLMC \eqref{eqn:algorithm}, that approximates \eqref{eqn:anti:reflected}, which yields 
iteration complexities for \eqref{eqn:algorithm}.
Our contributions can be stated as follows:
\begin{itemize}
\item 
We propose SRNLD, a continuous-time non-reversible Langevin SDE
on a constrained domain.
This includes the reversible Langevin SDE on a constrained domain
in the literature as a special case \cite{bubeck2015finite, bubeck2018sampling,Lamperski2021}.
Our technical novelty lies upon the construction of a skew-reflected boundary and the establishment of the well-posedness of the non-reversible Langevin SDE with a skew-reflected boundary.
First, we show the existence of the Skorokhod problem (Lemma~\ref{lem:Skorokhod}).
Next, we show that SRNLD admits
the Gibbs distribution as an invariant distribution (Theorem~\ref{thm:Gibbs}).
Moreover, we obtain non-asymptotic convergence rate
for continuous-time SRNLD in TV and $1$-Wasserstein distances to the Gibbs distribution (Theorem~\ref{thm:TV}).
We allow the target distribution to be non-convex, 
and our assumption is weaker than even the special case of
the reversible reflected Langevin SDE in the literature \cite{bubeck2015finite, bubeck2018sampling,Lamperski2021}. 
Furthermore, by breaking reversibility, 
we show that the non-reversible SRNLD 
has better convergence rate compared
to the reversible reflected Langevin SDE in the literature (Theorem~\ref{thm:TV}).
In the special case of quadratic objectives, 
using synchronous coupling and a novel weighted matrix norm, we obtain more explicit convergence rate (Proposition~\ref{prop:quadratic}).
\item
Moreover, we provide non-asymptotic discretization error
in $1$-Wasserstein distance for the discrete-time  
algorithm SRNLMC that keeps track
of the continuous-time SRNLD (Corollary~\ref{cor:discretization}). 
In the presence of skew-reflected boundary, 
we establish a novel estimate on the local time (Lemma~\ref{lemma: local time}), which is a key
ingredient in the proof of our discretization error bounds.
In particular,
by combining with our non-asymptotic analysis for the continuous-time SRNLD (Theorem~\ref{thm:TV}), 
this yields the iteration complexity for SRNLMC algorithm
in $1$-Wasserstein distance (Theorem~\ref{thm:final}, Corollary~\ref{cor:final}), better than PLMC algorithm in the literature that is based
on the reversible dynamics.
Hence, non-reversibility helps with the acceleration
in the context of constrained sampling.
\item
Finally, we provide numerical experiments to show
the efficiency of our proposed algorithm.
Our numerical experiments are conducted
using both synthetic data and real data. 
In particular, we start with a toy example
of sampling the truncated standard multivariate normal distribution.
Then, we conduct constrained Bayesian linear regression
and constrained Bayesian logistic regression using synthetic data.
Finally, we apply constrained Bayesian logistic regression 
to real datasets. Our numerical results indicate
that by appropriately choosing the anti-symmetric matrix
and the skew-projection, the proposed algorithm can outperform
the PLMC algorithm in the existing literature.
\end{itemize}

Comparing with the literature on constrained sampling using Langevin algorithms, \cite{bubeck2015finite, bubeck2018sampling,Lamperski2021} 
are the most relevant to our paper. When anti-symmetric matrix $J(\cdot)\equiv 0$, our algorithm SRNLMC covers PLMC as a special case.
Even though our model is more general, our technical assumptions are
weaker. Instead of assuming that the target function is log-concave or strongly log-concave as in \cite{bubeck2015finite, bubeck2018sampling,Lamperski2021}, we relax this requirement by only imposing the existence of a spectral gap. The precise conditions for this are detailed in Assumption \ref{assump: spectral gap}, followed by detailed discussions afterwards.

The paper is organized as follows. We state the main results in Section~\ref{sec:main}. In particular, we show the existence of the Skorokhod problem in Section~\ref{sec:Skorokhod}. In Section~\ref{sec:continuous}, we study the continuous-time
dynamics, and show that the Gibbs distribution
is an invariant distribution, and provide a non-asymptotic convergence result in TV and $1$-Wasserstein distances. We provide discretization analysis in Section~\ref{sec:discrete} by obtaining discretization error in $1$-Wasserstein distance,
and as a consequence, the iteration complexity of the discrete-time algorithm. Finally, numerical experiments are provided in Section~\ref{sec:numerical}.
 
%%%%%%%%%%%%%%%%%%%%%%%%%%
\section{Main Results}\label{sec:main}
%%%%%%%%%%%%%%%%%%%%%%%%%%%%%%%%%%%%%%%%%%%%
Throughout this work, we make the following assumptions for the domain
\(\mathcal{C}\), the target function $f$, 
and the skew function $J$:

\begin{assumption}[Domain Assumption]\label{assump:domain}
The domain \(\mathcal{C}\) is bounded, convex, and has a \(C^1\) boundary
\(\partial \mathcal{C}\). For the convenience of the analysis, we further assume that $0\in\mathcal{C}$ and $\mathcal{C}$ is contained in a ball centered at $0$ with radius $R>0$ and contains a ball centered at $0$ with radius $r>0$.
\end{assumption}

\begin{assumption}[Lipschitz Conditions]\label{assump:f:J}
The gradient of the target function \(\nabla f\) and the operator \(J\) satisfy the Lipschitz condition, i.e., there exist a constants \(L,L_{J} > 0\) such that:
\[
\|\nabla f(x) - \nabla f(y)\| \leq L \|x - y\| \quad \text{and} \quad \|J(x) - J(y)\| \leq L_{J} \|x - y\|,
\]
for all \(x, y \in \mathcal{C}\).
\end{assumption}

Note that Assumption~\ref{assump:domain} and Assumption~\ref{assump:f:J}
imply that for any $x\in\mathcal{C}$, $\Vert J(x)\Vert\leq\Vert J(0)\Vert+L_{J}\Vert x\Vert\leq\Vert J\Vert_{\infty}:=\Vert J(0)\Vert+L_{J}R$. Similarly, for any $x\in\mathcal{C}$, $\Vert\nabla f(x)\Vert\leq\Vert\nabla f\Vert_{\infty}:=\Vert\nabla f(0)\Vert+LR$.
This implies that $J\nabla f$ is $(\Vert J\Vert_{\infty}L+\Vert\nabla f\Vert_{\infty}L_{J})$-Lipschitz.

\subsection{Skorokhod Problem}\label{sec:Skorokhod}

In this section, we first show the existence of the Skorokhod problem with skew-unit inner normal vector corresponding to \eqref{eqn:anti:reflected}. 

\begin{lemma}\label{lem:Skorokhod}
For a non-reversible Langevin SDE takes the form of \eqref{eqn:anti}, there
exists a skew-reflected non-reversible Langevin dynamics in the form of
\eqref{eqn:anti:reflected}, and the soluiton is unique in the strong sense.  
\end{lemma}

\begin{proof}
The proof follows from solving the Skorokhod problem for the non-reversible Langevin SDE \eqref{eqn:anti}. The existence of the solution for the Skorokhod problem follows from \cite{tanaka1979stochastic}[Theorem 4.2], see also \cite{LionsSznitman}. In the current setting, our reflection is defined through a skew-unit vector $\nu^J$ in \eqref{skew unit norm vector}. To be precise, the discretization of the non-reversible SDE is kept inside the domain by skew-projection \eqref{defn: PCJ}:
\[
x_{k+1}=\mathcal P^J_{\mathcal C}\left(x_k-\eta(I+J(x_k))\nabla f(x_k)+\sqrt{2\eta}\xi_{k}\right).
\]
 The inner unit normal vector associated with the standard projection map $\mathcal P_{\mathcal C}$ is defined as, for $s=k\eta$:
\[
\nu_s=-\frac{\tilde{x}_{k+1}-\mathcal P_{\mathcal C}(\tilde{x}_{k+1})}{\Vert\tilde{x}_{k+1}-\mathcal P_{\mathcal C}(\tilde{x}_{k+1})\Vert}.
\]
where
\begin{equation}\label{tilde:x:k:plus:1}
\tilde{x}_{k+1}:=x_k-\eta(I+J(x_{k}))\nabla f(x_k)+\sqrt{2\eta}\xi_{k+1}.    
\end{equation}
We then define the inner skew-unit normal vector $\nu^J_s$ sharing the same origin on the boundary with $\nu_s$ as below:
\begin{equation*}
    \nu^J_s:= \frac{\mathcal R_{\mathcal C}^J(\tilde{x}_{k+1})-\mathcal P_{\mathcal C}(\tilde{x}_{k+1})}{\Vert\mathcal R^J_{\mathcal C}(\tilde{x}_{k+1})-\mathcal P_{\mathcal C}(\tilde{x}_{k+1})\Vert},
\end{equation*}
where $\tilde{x}_{k+1}$ is defined in \eqref{tilde:x:k:plus:1}.
Following from our definition, we observe that 
\begin{equation*}
    \nu^J(X_s) = \frac{(I+J(X_{s}))\nu_s}{\sqrt{\|\nu_s\|^2+\|J\nu_s\|^2}},
\end{equation*}
which ensures that $\mathcal P_{\mathcal C}^J$ keeps the trajectory inside the
domain $\mathcal C$. Since $x_k-\eta(I+J(x_{k}))\nabla f(x_k)$ are constant on
each of the time interval $[k\eta,(k+1)\eta)$, we are thus reduced to solve a Skorokhod
problem with skew-reflection on the boundary. In particular, the existence of
solution for the SDE with skew-reflection following from
\cite{tanaka1979stochastic}[Remark 2.1]. The unit vector $\nu_s^J$ does not need
to be normal, but satisfying $\langle \nu_t^J-X_{t},L(dt)\rangle\ge 0$, which is
true following from our definition of $\nu^J$. We thus show the existence of the
solution. Under Assumption \ref{assump:f:J}, and following the results in \cite{dupuis1993sdes}, which establish the properties of SDEs with oblique reflection in non-smooth domains, we deduce the uniqueness of the strong solution.
\end{proof}

\subsection{Continuous-time analysis}\label{sec:continuous}

%{\color{red}For our continuous-time SDE \eqref{eqn:anti:reflected}, we need to add some conditions to make sure there exists a unique solution. If the solution is unique, does that mean Gibbs distribution is the unique invariant distribution?}

%{\color{red}For unconstrained case, it is assumed in \cite{HHS05} the following two conditions: (1) $\nabla f$ is in $L^{1}(\pi)\cap L^\ell_{loc}(\pi)$
%or some $\ell>d$. (2) $\lambda_{0}<0$.
%Since $\mathcal{C}$ is assumed to be bounded, we just need $\nabla f$ to be continuous and then these assumption (1) automatically hold right? Please double check what conditions we need.
%}

%{\color{red}Another thing is that for continuous-time analysis, as well as the well-posedness of reflected SDE, what assumption we need to assume for $\mathcal{C}$?}

%{\color{purple}Using the condition in \cite[(2.20)] {costantini1992skorohod} for
%SDEs with oblique boundary condition, the existence and uniqueness of the (weak)
%solutions are shown in \cite[Theorem 5.4]{costantini1992skorohod}. The
%well-posedness works for convex $\mathcal{C}$ with $C^1$-bounary. For the
%uniqueness of the invariant measure, it solves the $\mathcal{L}^* \pi=0$ where
%$\mathcal{L}^*$ is the adjoint operator of $\mathcal{L}$. The uniqueneness could
%be reduced to the maximum principle in \cite{lieberman2020gradient}. I think we
%should use the condition $L^\ell$ because continuous assuption is stronger hence
%it will lead to a weaker result.  }

Our first main result is that the Gibbs distribution constrained on $\mathcal{C}$
is an invariant measure for the $X_{t}$ process in \eqref{eqn:anti:reflected}.
Before we proceed, 
we derive the infinitesimal generator for $X_{t}$ process in \eqref{eqn:anti:reflected}.

\begin{lemma}\label{lem:generator}
The infinitesimal generator $\mathcal{L}$  
for $X_{t}$ process in \eqref{eqn:anti:reflected} is given as follows.
For any $g\in\mathcal{D}(\mathcal{L})$:
\begin{equation}\label{eqn:generator}
\mathcal{L}g:=-\langle\nabla g,(I+J)\nabla f\rangle+\Delta g,
\end{equation}
subject to the Neumann boundary condition:
\begin{equation}\label{eqn:boundary}
\nabla g\cdot \nu^J=0.
% \nabla g(x)\cdot\nu(x)=\left(J\nabla g(x)\right)\cdot\nu(x),\qquad\text{for any  $x\in\partial\mathcal{C}$}.
\end{equation}
\end{lemma}

\begin{remark}
Notice that, the Neumann boundary condition \eqref{eqn:boundary} is equivalent to the following condition,
  \begin{equation}
      \nabla g(x)\cdot\nu(x)=\left(J(x)\nabla g(x)\right)\cdot\nu(x),\qquad\text{for any  $x\in\partial\mathcal{C}$}.
  \end{equation}
\end{remark}

\begin{proof}[Proof of Lemma~\ref{lem:generator}]
Let us prove that the infinitesimal generator for $X_{t}$ process in \eqref{eqn:anti:reflected} is given by 
$\mathcal{L}$ \eqref{eqn:generator} with the boundary condition \eqref{eqn:boundary}.
  Let $P_t$ be the semigroup associated with $X_{t}$, where $X_{t}$ satisfies \eqref{eqn:anti:reflected}.
  In other words, for given $g\in \mathcal{D}(\mathcal{L})$, 
  $$P_t g(x):=\mathbb{E}^x \left[g(X_{t})\right]=\mathbb{E}\left[ g(X_{t})|X_0=x \right].$$
By It\^{o}'s formula, we have that 
\begin{align}
  g(X_{t})-g(x)&=\int_{0}^{t}\langle \nabla g(X_{s}), -(I+J(X_{s}))\nabla f(X_{s}) \rangle ds +\int_{0}^{t}\sqrt{2}\langle \nabla g(X_{s}), dW_s \rangle\nonumber\\
  &\qquad\qquad+\int_{0}^{t}\Delta g(X_{s}) ds+\int_{0}^{t}\langle \nabla g(X_{s}), \nu^J(X_{s})\rangle L(ds),\label{eqn:RHS}
\end{align}
where the second term on the right hand side of \eqref{eqn:RHS} is a martingale and the last term in \eqref{eqn:RHS} vanishes by the boundary condition for $g$ (see \eqref{eqn:boundary}). Then we get that 
$$\mathcal{L}g:=\lim_{t\downarrow 0}\frac{P_t g-g}{t}=\lim_{t\downarrow 0}\frac{\mathbb{E}^x \left[g(X_{t})\right]-g(x)}{t}=-\langle -\nabla g, (I+J)\nabla f \rangle+\Delta g,$$
which verifies \eqref{eqn:generator}.
\end{proof}

Now, we are ready to state our first main result.

\begin{theorem}\label{thm:Gibbs}
The Gibbs distribution $\pi\propto e^{-f(x)}$, $x\in\mathcal{C}$ is an 
invariant measure for the $X_{t}$ process in \eqref{eqn:anti:reflected}.
\end{theorem}

\begin{proof}
First, we can write
\begin{equation}
d\pi(x)=\frac{1}{Z}e^{-f(x)}dx,\qquad x\in\mathcal{C},
\end{equation}
where 
\begin{equation}\label{defn:Z}
Z=\int_{\mathcal{C}}e^{-f(x)}dx
\end{equation}
is the normalizing constant.

Next, we recall that the $X_{t}$ process in \eqref{eqn:anti:reflected}
has the infinitesimal generator $\mathcal{L}$ as given in \eqref{eqn:generator}.
For any $g\in\mathcal{D}(\mathcal{L})$, we can compute that
\begin{align}
&\int_{\mathcal{C}}\mathcal{L}g(x)d\pi(x)
\nonumber
\\
&=-\int_{\mathcal{C}}\langle\nabla g(x),(I+J(x))\nabla f(x)\rangle d\pi(x)
+\int_{\mathcal{C}}\Delta g(x)d\pi(x)
\nonumber
\\
&=\frac{1}{Z}\left(-\int_{\mathcal{C}}\langle\nabla g(x),(I+J(x))\nabla f(x)\rangle e^{-f(x)}dx
+\int_{\mathcal{C}}\Delta g(x)e^{-f(x)}dx\right)
\nonumber
\\
&=\frac{1}{Z}\left(-\int_{\mathcal{C}}\langle(I+J^{\top}(x))\nabla g(x),\nabla f(x)\rangle e^{-f(x)}dx
+\int_{\mathcal{C}}\langle\nabla,(I+J^{\top}(x))\nabla g(x)\rangle e^{-f(x)}dx\right)
\nonumber
\\
&=\frac{1}{Z}\int_{\mathcal{C}}\nabla\cdot\left((I+J^{\top}(x))\nabla g(x)e^{-f(x)}\right)dx
\nonumber
\\
&=0,
\end{align}
where we used the boundary condition \eqref{eqn:boundary}, and the fact
$\langle a,Mb\rangle=\langle M^{\top}a,b\rangle$
for any matrix $M\in\mathbb{R}^{d\times d}$ and vectors $a,b\in\mathbb{R}^{d}$
and the fact that $J$ is an anti-symmetric matrix such that $J^{\top}=-J$
and thus $\langle\nabla, J^{\top}\nabla g\rangle=0$.
Hence,  $\pi\propto e^{-f(x)}$, $x\in\mathcal{C}$ is an 
invariant measure for the $X_{t}$ process in \eqref{eqn:anti:reflected}.
\end{proof}

Indeed, we can establish the following commutator identity, 
which we will see later that it is a stronger result than Theorem~\ref{thm:Gibbs} 
which implies Theorem~\ref{thm:Gibbs}.

\begin{lemma}
[Commutator Identity]
  \label{lemma:commutator}
  Given $g, h\in \mathcal{D}(\mathcal{L})$, it holds that 
  \begin{equation}
    \int_{\mathcal{C}}h \mathcal{L} g\; d\pi-\int_{\mathcal{C}}g\mathcal{L}h \;d\pi=\int_{\mathcal{C}}\nabla\cdot (hJ\nabla g-gJ\nabla h)d\pi.
    \label{eqn:identity}
  \end{equation}
\end{lemma}

\begin{proof}
  \begin{align*}
    \int_{\mathcal{C}}h\mathcal{L}gd\pi&=\frac{1}{Z}\int_{\mathcal{C}}h\mathcal{L}g e^{-f}dx\\
    &=\frac{1}{Z}\int_{\mathcal{C}}h\left( -\langle \nabla g, (I+J)\nabla f \rangle+\Delta g \right)e^{-f}dx\\
    &=\frac{1}{Z}\int_{\mathcal{C}}h\left( -\langle (I-J)\nabla g, \nabla f \rangle+\Delta g \right)e^{-f}dx\\
    &=\frac{1}{Z}\int_{\mathcal{C}}h \langle (I-J)\nabla g, \nabla e^{-f} \rangle +(h\Delta g)e^{-f}dx\\
    &=-\frac{1}{Z}\int_{\mathcal{C}}\langle \nabla h, \nabla g \rangle e^{-f} dx +\frac{1}{Z}\int_{\mathcal{C}}\nabla\cdot (h J\nabla g) e^{-f} dx\\
    &=-\int_{\mathcal{C}}\langle \nabla h, \nabla g \rangle d\pi+\int_{\mathcal{C}}\nabla \cdot (hJ \nabla g)d\pi,
  \end{align*}
  where $Z>0$ is the normalizing constant given in \eqref{defn:Z}.
  
  By swapping $h$ and $g$, we get the identity \eqref{eqn:identity}.
\end{proof}

\begin{remark}
Lemma~\ref{lemma:commutator} is an extension of Theorem~\ref{thm:Gibbs} in the sense that it implies Theorem~\ref{thm:Gibbs}.
To see this, since $\nabla\cdot J=0$, by letting $h\equiv 1$, we get that 
\begin{align}
\int_{\mathcal{C}}\mathcal{L}g d\pi&=\int_{\mathcal{C}}\nabla\cdot (J\nabla g) d\pi\nonumber\\
&=\int_{\mathcal{C}}\sum_{i=1}^{d}\sum_{j=1}^{d}\nabla_i(J_{ij}\nabla_j g) d\pi\nonumber\\
&=\int_{\mathcal{C}}\sum_{i=1}^{d}\sum_{j=1}^{d}\nabla_i J_{ij}\nabla_j g d\pi+\int_{\mathcal{C}}\sum_{i=1}^{d}\sum_{j=1}^{d}J_{ij}\nabla_{ij}g d\pi\label{two:terms}\\
&=0.\nonumber
\end{align}
The first term in \eqref{two:terms} is zero since $\nabla\cdot J=0$ and the second term in \eqref{two:terms} is also zero because the anti-symmetry property of $J$ and the symmetry of $\nabla^2 g$; more precisely, 
\begin{equation*}
\sum_{i=1}^{d}\sum_{j=1}^{d}J_{ij}\nabla_{ij}g=\sum_{i=1}^{d}\sum_{j=1}^{d}J_{ij}\nabla_{ji}g=\sum_{i=1}^{d}\sum_{j=1}^{d}J_{ji}\nabla_{ij}g=-\sum_{i=1}^{d}\sum_{j=1}^{d}J_{ij}\nabla_{ij}g=0.
\end{equation*}
\end{remark}

\begin{remark}
Since $\nabla\cdot  J=0$, by letting $h=g$, we have the Dirichlet form:
\begin{equation}\label{defn:Dirichlet}
\mathcal{E}(g,g):=-\int_{\mathcal{C}}g\mathcal{L}g d\pi=\int_{\mathcal{C}}\Vert\nabla g\Vert^2 d\pi.
\end{equation}
\end{remark}

Next, let us study the convergence speed
of the continuous-time process $X_{t}$ in \eqref{eqn:anti:reflected}
to its invariant distribution.
Before we proceed, let us introduce two notions
of non-asymptotic convergence bound here.

First, we let $\lambda_{J}$ be the spectral gap
of $\mathcal{L}$ in $L^{2}(\pi)$, 
where the subscript emphasizes the dependence
on the anti-symmetric matrix $J$.
That is, 
\begin{equation}
\int_{\mathcal{C}}\left(\mathbb{E}^{x}\left[g(X_{t})\right]-\pi(g)\right)^{2}\pi(x)dx
=\Vert P_{t}g-\pi(g)\Vert^{2}
\leq
C_{J}\Vert g-\pi(g)\Vert^{2}e^{-2\lambda_{J}t},
\end{equation}
for some constant $C_{J}$.
In particular, $\lambda_{0}$ denotes
the spectral gap when $J=0$, i.e. 
the case of the overdamped Langevin diffusion.
We assume the existence of the spectral gap of $\mathcal{L}$ when $J=0$, i.e. it is positive.

\begin{assumption}\label{assump: spectral gap}
Assume that $\lambda_{0}>0$.
\end{assumption}

It is worth noting that the existence of a spectral gap can be derived from the
logarithmic Sobolev inequality, as discussed in \cite{bakry2014analysis}.
Furthermore, the logarithmic Sobolev inequality itself can be established under
the weaker condition of the existence of a Lyapunov function for the diffusion
generator \cite{bakry2008simple, cattiaux2010note}. Explicit lower bounds for
the logarithmic Sobolev inequality in convex bounded domain have been explored
in works such as \cite{wang1997estimation}.

Next, we introduce the rate of convergence of $p(t,x,y)$,
the transition probability density of $X_{t}$, to the Gibbs distribution $\pi$ in the variational norm, which is equivalent to the TV distance
up to a $1/2$ factor. Let us define:
\begin{equation}
\rho_{J}:=\sup\left\{\rho:\int_{\mathcal{C}}|p(t,x,y)-\pi(y)|dy\leq g(x)e^{-\rho t}\right\},
\end{equation}
for some $g(x)$ that may depend on $J$. In particular, $\rho_{0}$ denotes
the case when $J=0$, i.e. the case of the overdamped Langevin diffusion.

We will obtain a non-asymptotic bound
for the convergence of the distribution of $X_{t}$
to the invariant distribution $\pi$ in total variation (TV) distance
and $1$-Wasserstein distance.

\begin{theorem}\label{thm:TV}
There exists some constant $\mathcal{K}>0$ such that
for any $X_{0}=x\in\mathcal{C}$,
\begin{equation}
\mathrm{TV}(\mathrm{Law}(X_{t}),\pi)
\leq\mathcal{K}e^{-\rho_{J}t},
\end{equation}
where $\rho_{J}\geq\rho_{0}>0$. In addition, 
\begin{equation}
\mathcal{W}_{1}(\mathrm{Law}(X_{t}),\pi)
\leq 2R\cdot\mathcal{K}e^{-\rho_{J}t},
\end{equation}
\end{theorem}

Theorem~\ref{thm:TV} is inspired by the result for the convergence of \eqref{eqn:anti} to the invariant
distribution for the unconstrained domain 
in \cite{HHS05}. Since our $\mathcal{C}$ is bounded, we are able to obtain
a convergence bound in Theorem~\ref{thm:TV} uniformly in the initial starting point $X_{0}=x\in\mathcal{C}$.
Moreover, in addition to the TV distance, we also have the $1$-Wasserstein guarantee in Theorem~\ref{thm:TV}.
To prove Theorem~\ref{thm:TV}, we will establish a sequence
of technical lemmas.

\begin{lemma}\label{lem:lambda}
$\lambda_{J}\geq\lambda_{0}$.    
\end{lemma}

\begin{proof}[Proof of Lemma~\ref{lem:lambda}]
We follow the similar argument
as in the proof of Theorem~1 in \cite{HHS05} for the unconstrained case.
If $g\in\mathcal{D}(\mathcal{L})$, then $g\in\mathcal{D}(\mathcal{E})$ and
\begin{equation}\label{key:ineq}
\mathcal{E}(g,g)\leq-\int_{\mathcal{C}}(\mathcal{L}g)gd\pi,    
\end{equation}
(see page 124 in \cite{stannat1999nonsymmetric}), 
where $\mathcal{E}(\cdot,\cdot)$ is the Dirichlet form in \eqref{defn:Dirichlet}.
Indeed, \eqref{key:ineq} holds in a more general setting \cite{stannat1999nonsymmetric}, and in our case, 
the inequality becomes the equality in \eqref{key:ineq} (see our Equation~\eqref{defn:Dirichlet}).
For any $g$ with $\Vert g\Vert=1$ and $\pi(g)=0$, let $h(t):=\Vert P_{t}g\Vert^{2}$. 
Then $h(0)=1$ and by \eqref{key:ineq}, 
we have
\begin{equation}
h'(t)=2\int_{\mathcal{C}}(\mathcal{L}P_{t}g)(P_{t}g)d\pi
\leq -2\mathcal{E}(P_{t}g,P_{t}g)
\leq -2\lambda_{0}h(t).
\end{equation}
This implies that the operator norm $\Vert P_{t}\Vert$ in the space $\{g:g\in L^{2}(\pi),\pi(g)=0\}$ is less than or equal to $e^{-\lambda_{0}t}$. Hence, $\lambda_{J}\geq\lambda_{0}$.
This completes the proof.
\end{proof}

Next, we show the following lemma.

%{\color{red}for unconstrained case, in \cite{HHS05}, it is assumed $\nabla f$ is
%locally bounded. Since $\mathcal{C}$ is bounded, I think it automatically holds
%in our case? plus, in \cite{HHS05}, $\mathcal{K}$ below depends on $x$. Since
%$\mathcal{C}$ is bounded, I suppose we can have $\mathcal{K}$ uniform in $x$.
%Please double check.}

%{
%\color{purple} It is true, we can quote \cite{lieberman1990holder} in which it
%%is established that $\nabla f$ is H\"{o}lder continuous. 
%}

\begin{lemma}\label{lem:TV}
There exists some constant $\mathcal{K}>0$ such that
for any $X_{0}=x\in\mathcal{C}$,
\begin{equation}
\mathrm{TV}(\mathrm{Law}(X_{t}),\pi)
\leq \mathcal{K}e^{-\rho_{J}t},
\end{equation}
where $\rho_{J}\geq\lambda_{J}$.
\end{lemma}

\begin{proof}[Proof of Lemma~\ref{lem:TV}]
The proof is inspired by that of Theorem~4 in \cite{HHS05} 
for the unconstrained case. 
Let $p_{t}(x,y):=p(t,x,y)/\pi(y)$ for any $x,y\in\mathcal{C}$, 
where $p(t,x,y)$ denotes the probability density function
for $X_{t}$ given that $X_{0}=x\in\mathcal{C}$. 
By the definition of the TV distance, we have
\begin{equation}
\mathrm{TV}(\mathrm{Law}(X_{t}),\pi)
=\int_{\mathcal{C}}|p(t,x,y)-\pi(y)|dy
=\int_{\mathcal{C}}|p_{t}(x,y)-1|\pi(y)dy.
\end{equation}
By letting $p^{\ast}_{t}(\cdot,\cdot)$ denote the adjoint process, 
we can compute that for any $t\geq 1$,
\begin{align*}
\int_{\mathcal{C}}|p_{t}(x,y)-1|\pi(y)dy
&=\int_{\mathcal{C}}\left|\int_{\mathcal{C}}\left(p_{1}(x,z)p_{t-1}(z,y)-1\right)\pi(z)dz\right|\pi(y)dy
\\
&=\int_{\mathcal{C}}\left|\int_{\mathcal{C}}(p_{1}(x,z)p^{\ast}_{t-1}(y,z)-1)\pi(z)dz\right|\pi(y)dy
\\
&=\int_{\mathcal{C}}\left|\int_{\mathcal{C}}p^{\ast}_{t-1}(y,z)p_{1}(x,z)\pi(z)dz-\int_{\mathcal{C}}p_{1}(x,z)\pi(z)dz\right|\pi(y)dy
\\
&=\int_{\mathcal{C}}|P^{\ast}_{t-1}(p_{1}(x,\cdot))(y)-\pi(p_{1}(x,\cdot))|\pi(y)dy
\\
&\leq\left(\int_{\mathcal{C}}|P^{\ast}_{t-1}(p_{1}(x,\cdot))(y)-\pi(p_{1}(x,\cdot))|^{2}\pi(y)dy\right)^{1/2}
\\
&\leq C\Vert p_{1}(x,\cdot)-1\Vert e^{-\lambda_{J}t},
\end{align*}
where we used Cauchy-Schwarz inequality and $C$ is some constant.  By applying the regularity estimates for oblique parabolic equations from \cite{lieberman1990holder}, we conclude that \(p_1(x, \cdot)\) is Hölder continuous on \(\mathcal{C}\). Consequently, \(p_1(x, \cdot) \in L^2(\pi)\), and thus, the final inequality holds.
This also establishes that $\rho_{J}\geq\lambda_{J}$. 
\end{proof}

Next, we show that $\lambda_{0}=\rho_{0}$.

\begin{lemma}\label{lem:lambda:0:rho:0}
$\lambda_{0}=\rho_{0}$.   
\end{lemma}

\begin{proof}[Proof of Lemma~\ref{lem:lambda:0:rho:0}]
The proof is inspired by that of Theorem~5 in \cite{HHS05} 
for the unconstrained case.
When $J=0$, $X_{t}$ is reversible and $P_{t}$ is self-adjoint in $L^{2}(\pi)$.
Given $g$ with $\pi(g)=0$ and $g\in C^{\infty}$, we have
\begin{align*}
\Vert P_{t}g\Vert^{2}&=\pi\left(gP_{2t}g\right)    
\\
&=\int_{\mathcal{C}}g(x)\left(\int_{\mathcal{C}}(p_{2t}(x,y)-1)g(y)\pi(y)dy\right)\pi(x)dx
\\
&\leq\sup_{x\in\mathcal{C}}\Vert g(x)\Vert^{2}\int_{\mathcal{C}}\int_{\mathcal{C}}|p_{2t}(x,y)-1|\pi(y)dy\pi(x)dx
\\
&\leq\sup_{x\in\mathcal{C}}\Vert g(x)\Vert^{2}\int_{\mathcal{C}}\mathcal{K}e^{-2\rho t}\pi(x)dx
\\
&=\sup_{x\in\mathcal{C}}\Vert g(x)\Vert^{2}\mathcal{K}e^{-2\rho t},
\end{align*}
where we applied Lemma~\ref{lem:TV}.

By Lemma~2.2. in \cite{Rockner2001}, for $s\leq t$ and $\pi(g^{2})=1$, 
\begin{equation}
\Vert P_{s}g\Vert^{2}\leq\left(\Vert P_{t}g\Vert^{2}\right)^{s/t}\leq   
\left(\sup_{x\in\mathcal{C}}\Vert g(x)\Vert^{2}\mathcal{K}\right)^{s/t}e^{-2\rho_{0}s},
\end{equation}
which implies 
\begin{equation}\label{take:derivative:eqn}
h(s)\leq\left(\sup_{x\in\mathcal{C}}\Vert g(x)\Vert^{2}\mathcal{K}\right)^{s/t}e^{-2\rho_{0}s},
\end{equation}
where $h(t):=\Vert P_{t}g\Vert^{2}$.
Since $h(0)=1$, the equality holds at $s=0$ in \eqref{take:derivative:eqn}. By taking derivative in \eqref{take:derivative:eqn} with respect to $s$
and letting $s=0$, we get
\begin{equation}
-2\mathcal{E}(g,g)\leq\frac{1}{t}\log\left(\sup_{x\in\mathcal{C}}\Vert g(x)\Vert^{2}\mathcal{K}\right)-2\rho_{0},
\end{equation}
where $\mathcal{E}(g,g)$ is the Dirichlet form \eqref{defn:Dirichlet}, and we used the fact
that $h'(0)=-2\mathcal{E}(g,g)$.
By letting $t\rightarrow\infty$, we have
\begin{equation}
\mathcal{E}(g,g)\geq\rho_{0}.
\end{equation}
This completes the proof.
\end{proof}

Now, we are finally ready to prove Theorem~\ref{thm:TV}.

\begin{proof}[Proof of Theorem~\ref{thm:TV}]
It follows from Lemma~\ref{lem:TV}
that there exists some constant $\mathcal{K}>0$ such that
for any $X_{0}=x\in\mathcal{C}$,
\begin{equation}
\mathrm{TV}(\mathrm{Law}(X_{t}),\pi)
\leq \mathcal{K}e^{-\rho_{J}t},
\end{equation}
where $\rho_{J}\geq\lambda_{J}$.
We recall from Lemma~\ref{lem:lambda}
that $\lambda_{J}\geq\lambda_{0}$
and from Lemma~\ref{lem:lambda:0:rho:0}
that $\lambda_{0}=\rho_{0}$. 
Hence, we conclude that $\rho_{J}\geq\lambda_{J}\geq\lambda_{0}=\rho_{0}>0$, where we also used $\lambda_{0}>0$ from Assumption~\ref{assump: spectral gap}.
Finally, we have the inequality
\begin{equation}
\mathcal{W}_{1}(\mathrm{Law}(X_{t}),\pi)
\leq 2R\cdot\mathrm{TV}(\mathrm{Law}(X_{t}),\pi),
\end{equation}
which is due to the fact $1$-Wasserstein distance can be bounded by TV distance on a bounded domain \cite{GibbsSu2002}.
The proof is complete.
\end{proof}

In Theorem~\ref{thm:TV}, we do not have an explicit bound
on $\rho_{J}$. In general, it seems hard to show how much acceleration you can obtain by comparing $\rho_{J}$ to $\rho_{0}$. 
The following lemma will play a key role in quantifying the acceleration rate in the upcoming Proposition~\ref{prop:quadratic}.

\begin{lemma}\label{lemma:matrix}
Let \( J \in \mathbb{R}^{n \times n} \) be antisymmetric, and define
\[
A := (I + J)^{-1}, \quad S := \frac{1}{2}(A + A^\top).
\]
Then the symmetric part \( S \) satisfies \( S \preceq I \). Moreover, if \( J \neq 0 \), then \( S \prec I \).
\end{lemma}

\begin{proof}
First, we can compute that
\[
A^\top = (I + J)^{-\top} = (I - J)^{-1}, \quad \text{and} \quad S = \frac{1}{2}\left((I + J)^{-1} + (I - J)^{-1}\right).
\]
By using the identity
\[
(I + J)^{-1} + (I - J)^{-1} = 2(I - J^2)^{-1},
\]
we obtain \( S = (I - J^2)^{-1} \). Since \( J \) is antisymmetric, \( J^2 \preceq 0\) such that \(I - J^2 \succeq I \). By operator monotonicity of the matrix inverse \cite{toda2011operator},
\[
S = (I - J^2)^{-1} \preceq I.
\]
If \( J \neq 0 \), then \( J^2 \prec 0\) which implies that $I - J^2 \succ I$ and hence $S \prec I$.
\end{proof}

In what follows, we give an example when $f(x)=x^{\top}Hx$ and $H,J$ are constant matrices.

\begin{proposition}\label{prop:quadratic}
For $f(x)=x^{\top}Hx$, and $H$, $J$ are constant positive semi-definite and anti-symmetric matrices respectively, we have 
\begin{equation}
\mathcal{W}_{1}(\mathrm{Law}(X_{t}),\pi)\leq 2c_{(I+J)^{-1}}^{-\frac12}Re^{-\tilde \lambda t},
\end{equation}
where $\tilde\lambda:=\lambda C^{-1}_{(I+J)^{-1}}$ with $\lambda$ being the smallest eigenvalue of matrix $H$ 
and $c_{(I+J)^{-1}}, C_{(I+J)^{-1}}\le 1$ denote the smallest and largest eigenvalues of the symmetric part of $(I+J)^{-1}$, i.e., $\frac12 [(I+J)^{-1}+(I+J)^{-\top}]$.
\end{proposition}

\begin{proof}
We generalize the synchronous coupling method under a weighted matrix norm to establish the proof.
Consider another coupled process $\tilde X_t$ being a copy of $X_t$ as defined in \eqref{eqn:anti:reflected},  driven by the same Brownian motion, satisfying the following dynamic, 
\begin{equation}\label{tilde X}
    d\tilde X_t= -(I+J)\nabla f(\tilde X_t)dt +\sqrt{2}dW_t+\nu^J(\tilde X_t)\tilde L(dt).
\end{equation}
    Consider the difference equation for $X_t-\tilde X_t$ defined by \eqref{eqn:anti:reflected} and \eqref{tilde X} with constant matrix $J$ and $f(x)=x^{\top}Hx$, we have the following differential equation  
\begin{equation}\label{difference equation}
\begin{split}
d(X_{t}-\tilde X_t)&=-[(I+J)\nabla f(X_{t})-(I+J)\nabla f(\tilde X_{t})]dt+\nu^J(X_{t})L(dt)-\nu^J(\tilde X_{t})\tilde L(dt)\\
&=-(I+J)H(X_t-\tilde X_t)dt+\nu^J(X_{t})L(dt)-\nu^J(\tilde X_{t})\tilde L(dt).
\end{split}
\end{equation}
For the constant anti-symmetric matrix $J$, we know $I+J$ is invertible since $\langle (I+J)x,x\rangle =\|x\|^2>0$ for all nonzero vector $x$.
Applying \eqref{difference equation} below for $\|X_{t}-\tilde X_t\|^2_{(I+J)^{-1}}:=(X_{t}-\tilde X_t)^{\top}(I+J)^{-1}(X_{t}-\tilde X_t)$ , we get
\begin{equation}
\begin{split}
&d\|X_{t}-\tilde X_t\|^2_{(I+J)^{-1}}
\\
&=-2(X_{t}-\tilde X_t)^{\top}(I+J)^{-1}(I+J)H(X_{t}-\tilde X_t)dt\\
&\qquad+(X_{t}-\tilde X_t)^{\top}(I+J)^{-1}\nu^J(X_{t})L(dt)+(\tilde X_{t}- X_t)^{\top}(I+J)^{-1}\nu^J(\tilde X_{t})\tilde L(dt)\\
&=-2(X_{t}-\tilde X_t)^{\top}H(X_{t}-\tilde X_t)dt\\
&\qquad+(X_{t}-\tilde X_t)^{\top} \frac{\nu_t}{\sqrt{\|\nu_t\|^2+\|J\nu_t\|^2}}L(dt)+(\tilde X_{t}- X_t)^{\top}\frac{\tilde \nu_t}{\sqrt{\|\tilde \nu_t\|^2+\|J\tilde \nu_t\|^2}}\tilde L(dt)\\
%&\le -2(X_{t}-\tilde X_{t})^{\top}(I+J)^{-1}(I+J)H(X_{t}-\tilde X_t)dt\\
&\le-2(X_t-\tilde{X}_t^\top)H(X_t-\tilde{X}_t)dt,
\end{split}
\end{equation}
where the last inequality follows from the fact $(X_{t}-\tilde X_t)^{\top}\nu_t=\langle X_{t}-\tilde X_t, \nu(X_{t})\rangle\le 0$, for $X_t\in \partial\mathcal C$, since $\nu$ is defined as the inner unit vector at $X_t$, and $X_t-\tilde X_t$ pointed outward, and similarly, $(\tilde X_{t}- X_t)^{\top}\tilde \nu_t\le 0$, for $\tilde X_t\in\partial\mathcal C.$
Denote $ \lambda\ge 0$ as the smallest eigenvalue for matrix $H$, we further get the following estimate,
\begin{equation}\label{apply:Gronwall}
   \frac{ d\|X_t-\tilde X_t\|^2_{(I+J)^{-1}} }{dt}\le -2  \lambda \|X_t-\tilde X_t\|^2\le -2\tilde{\lambda}\|X_t-\tilde{X}_t\|^2_{(I+J)^{-1}} 
\end{equation}
with $\tilde{\lambda}:=\lambda C^{-1}_{(I+J)^{-1}}$. Note that it follows from Lemma \ref{lemma:matrix} that $C_{(I+J)^{-1}}\le 1$.
By applying the Gr\"{o}nwall's inequality to \eqref{apply:Gronwall}, we conclude that
\begin{equation}
    \|X_t-\tilde X_t\|^2_{(I+J)^{-1}} \le e^{-2\tilde \lambda t}\|X_0-\tilde X_0\|^2_{(I+J)^{-1}},
\end{equation}
which further implies that,
\begin{equation}
    \|X_t-\tilde X_t\|\le c^{-\frac12}_{(I+J)^{-1}} \|X_t-\tilde X_t\|_{(I+J)^{-1}}\le c^{-\frac12}_{(I+J)^{-1}}e^{-\tilde \lambda t}\|X_0-\tilde X_0\|_{(I+J)^{-1}}.
\end{equation}
 By letting $\tilde X_{0}$ follow the invariant distribution $\pi$, and using
that fact that $\|X_{0}-\tilde X_{0}\|_{(I+J)^{-1}}\leq 2R$, we conclude that
\begin{equation}
\mathcal{W}_{1}(\mathrm{Law}(X_{t}),\pi)\leq 2c^{-\frac12}_{(I+J)^{-1}}Re^{-\tilde \lambda t}.
\end{equation}
This completes the proof.
\end{proof}

\subsection{Discretization analysis}\label{sec:discrete}

In this section, we quantify the discretization error between
the discrete-time algorithm SRNLMC \eqref{eqn:algorithm} and the continuous-time dynamics SRNLD \eqref{eqn:anti:reflected}.
We will provide explicit discretization error bound in terms of $1$-Wasserstein distance. 
Our discrete error analysis is inspired by that of PLMC \cite{bubeck2015finite} yet we need to establish some novel estimates due to the skew reflection in our algorithm.
%%%%%%%%%%%%%%%%%%%%%%%
We recall from \eqref{eqn:anti:reflected}
that $X_{t}$ is the continuous-time process satisfying
\begin{equation}\label{eqn:anti:reflected2}
dX_{t}=-(I+J(X_{t}))\nabla f(X_{t})dt+\sqrt{2}dW_{t}+\nu^J(X_{t})L(dt).
\end{equation}
Next, we let
\begin{equation}\label{Z:eqn}
Z_{t}=\sqrt{2}W_{t}-\int_{0}^{t}(I+J(X_{s}))\nabla f(X_{s})ds,
\end{equation}
so that $X_{t}$ solves the Skorokhod problem for $Z_{t}$.
We also let $\bar{Z}_{t}:=Z_{\eta\lfloor t/\eta\rfloor}$ be
the continuous interpolation of the discretization of $Z_{t}$, 
and $\bar{X}_{t}$ be the solution of the Skorokhod problem for $\bar{Z}_{t}$, 
that is, $\bar{X}_{t}$ is constant on any interval $[k\eta,(k+1)\eta)$
and for every $k\in\mathbb{N}$, 
\begin{equation}
\bar{X}_{(k+1)\eta}=\mathcal{P}_{\mathcal{C}}^{J}\left(\bar{X}_{k\eta}+Z_{(k+1)\eta}-Z_{k\eta}\right).
\end{equation}
We have the following technical lemma.

\begin{lemma}\label{lemma: local time}
Let $X_t$ be the solution to \eqref{eqn:anti:reflected2} subject to $X_0=x\in \mathcal{C}$. Then there exists a constant $C_{L} > 0$ such that for any $t>0$,
\[
\mathbb{E}\left[L(t)\right] \leq C_{L}+C_{L}(1+d) t.
\]
\end{lemma}

\begin{proof}
Let $h \in C^\infty(\overline{\mathcal{C}})$ be a smooth non-negative function such that near $\partial \mathcal{C}$,
\[
h|_{\partial \mathcal{C}} = 0, \quad \nabla h|_{\partial \mathcal{C}} = \nu,
\]
and there exists some universal constant $C>0$ such that $|h|\leq C$, $|\langle \nabla h, (I + J)\nabla f \rangle| \leq C$ and $|\Delta h| \leq 2Cd$. 
Applying Itô's formula to $h(X_t)$ yields
\begin{equation}\label{sub:into}
dh(X_t) = \langle \nabla h, dX_t \rangle + \frac{1}{2} \Delta h\, dt.
\end{equation}
Substituting the SDE \eqref{eqn:anti:reflected} into \eqref{sub:into} gives
\[
dh(X_t) = -\langle \nabla h, (I + J)\nabla f \rangle dt + \sqrt{2} \langle \nabla h, dW_t \rangle + \langle \nabla h, \nu^J \rangle L(dt) + \frac{1}{2} \Delta h\, dt.
\]
By using $|\langle \nabla h, (I + J)\nabla f \rangle| \leq C$, $|\Delta h| \leq 2Cd$, and $\langle \nabla h, \nu^J \rangle \geq \delta_0 > 0$ on $\partial \mathcal{C}$, we obtain
\begin{equation}\label{integrate:0:t}
dh(X_t) \geq \sqrt{2} \langle \nabla h, dW_t \rangle - C(1+d)\,dt + \delta_0\, L(dt).
\end{equation}
Integrating \eqref{integrate:0:t} from $0$ to $t$ gives
\begin{equation}\label{to:rearrange}
h(X_t) - h(x) \geq \sqrt{2} \int_0^{t} \langle \nabla h, dW_s \rangle - C(1+d) t + \delta_0\, L(t).
\end{equation}
By rearranging \eqref{to:rearrange}, we obtain:
\begin{equation}\label{to:take:exp}
L(t) \leq \frac{1}{\delta_0} \left( h(X_t) - h(x) -\sqrt{2} \int_0^{t} \langle \nabla h, dW_s \rangle + C(1+d) t \right).
\end{equation}
Taking expectation in \eqref{to:take:exp} and noting that the stochastic integral is a martingale, we obtain
\[
\mathbb{E}[L(t)] \leq \frac{1}{\delta_0} \left( \mathbb{E}[h(X_t) - h(x)] + C(1+d) t \right) \leq \frac{1}{\delta_{0}}\left(C+C(1+d) t\right)
=C_{L}+C_{L}(1+d) t,
\]
where $C_{L}:=C/\delta_{0}$.
This completes the proof.
\end{proof}

By using Lemma \ref{lemma: local time}, we obtain the following estimate.

\begin{lemma}\label{lem:1}
\begin{align}
\mathbb{E}\left[\int_{0}^{t}h_{\mathcal{C}}(-\nu^{J}(X_{s}))L(ds)\right] 
\leq
 R^2 C_{L} +R^2 C_{L}(1+d)t,
\end{align}   
where $h_{\mathcal{C}}$ is the support function of $\mathcal{C}$, 
i.e. $h_{\mathcal{C}}(y):=\sup\{\langle x,y\rangle:x\in\mathcal{C}\}$ for every $y\in\mathbb{R}^{d}$, and $C_{L}>0$ is a constant.
\end{lemma}

\begin{proof}
    According to Lemma \ref{lemma: local time} and the boundedness of $h_{\mathcal{C}}$ , we have 
    \begin{align*}   \int_{0}^{t}h_{\mathcal{C}}(-\nu^{J}(X_{s}))L(ds)\le \int_{0}^{t}|h_{\mathcal{C}}(-\nu^{J}(X_{s}))|L(ds) \le R^2 L(t).    
    \end{align*}
    By taking expectations of both sides, we get
    \begin{align}
\mathbb{E}\left[\int_{0}^{t}h_{\mathcal{C}}(-\nu^{J}(X_{s}))L(ds)\right] 
\leq R^2 C_{L} +R^2 C_{L}(1+d)t.
\end{align} 
\end{proof}

Let $\Vert\cdot\Vert_{\mathcal{C}}$ be the gauge
of $\mathcal{C}$, that is,
\begin{equation}
\Vert x\Vert_{\mathcal{C}}=\inf\{t\geq 0:x\in t\mathcal{C}\},
\qquad\text{for any $x\in\mathcal{C}$}.
\end{equation}
Next, we recall the definition of $Z_{t}$ from \eqref{Z:eqn}
and $\bar{Z}_{t}:=Z_{\eta\lfloor t/\eta\rfloor}$ is
the continuous interpolation of the discretization of $Z_{t}$.
We have the following technical lemma.

\begin{lemma}\label{lem:2}
There exists a universal constant $C>0$ such that
\begin{equation}
\mathbb{E}\left[\sup_{s\in[0,t]}\Vert Z_{s}-\bar{Z}_{s}\Vert_{\mathcal{C}}\right]
\leq CM_{\mathcal{C}}\sqrt{2d\eta}\sqrt{\log(t/\eta)}+\frac{\eta(1+\Vert J\Vert_{\infty})L}{r},
\end{equation}
where $M_{\mathcal{C}}:=\mathbb{E}[\Vert U\Vert_{\mathcal{C}}]$, where $U$ is uniformly
distributed on the sphere $\mathbb{S}^{d-1}$.
\end{lemma}

\begin{proof}
First of all, for any $x\in\mathcal{C}$, by Assumption~\ref{assump:domain} and Assumption~\ref{assump:f:J}, we have
\begin{equation}
\Vert (I+J(x))\nabla f(x)\Vert_{\mathcal{C}}
\leq\frac{1}{r}\Vert (I+J(x))\nabla f(x)\Vert
\leq\frac{(1+\Vert J\Vert_{\infty})L}{r}.
\end{equation}
Therefore, for any $t>0$, 
\begin{align}
\Vert Z_{t}-\bar{Z}_{t}\Vert_{\mathcal{C}}   
&\leq
\sqrt{2}\Vert W_{t}-\bar{W}_{t}\Vert_{\mathcal{C}}
+\int_{\eta\lfloor t/\eta\rfloor}^{t}\Vert (I+J(X_{s}))\nabla f(X_{s})\Vert_{\mathcal{C}}ds
\nonumber
\\
&\leq
\sqrt{2}\Vert W_{t}-\bar{W}_{t}\Vert_{\mathcal{C}}
+\frac{\eta(1+\Vert J\Vert_{\infty})L}{r},
\end{align}
where $\bar{W}_{t}:=W_{\eta\lfloor t/\eta\rfloor}$ be
the continuous interpolation of the discretization of 
the Brownian motion $W_{t}$.
It follows from Lemma~3 in \cite{bubeck2015finite} that
there exists a universal constant $C>0$ such that
\begin{equation}
\mathbb{E}\left[\sup_{s\in[0,t]}\Vert W_{s}-\bar{W}_{s}\Vert_{\mathcal{C}}\right]
\leq CM_{\mathcal{C}}\sqrt{d\eta}\sqrt{\log(t/\eta)}.
\end{equation}
This completes the proof.
\end{proof}

Next, we upper bound the $L_{1}$ distance between $X_{T}$ and $\bar{X}_{T}$.

\begin{lemma}\label{lem:discretization:final}
For any $T=K\eta\geq 0$,
\begin{align}
\mathbb{E}\Vert X_{T}-\bar{X}_{T}\Vert
&\leq\mathcal{E}_{\eta}(T):=\sqrt{2}\left(\sqrt{C}\sqrt{M_{\mathcal{C}}}(2d\eta)^{1/4}(\log(T/\eta))^{1/4}+\frac{\eta(1+\Vert J\Vert_{\infty})L}{r}\right)
\nonumber
\\
&\qquad\qquad\qquad
\cdot\left(R\sqrt{C_{L}}+\sqrt{T}\sqrt{C_{L}}\left(d+1\right)^{1/2}\right),\label{E:T:defn}
\end{align}
where $C>0$ is a universal constant.
\end{lemma}

\begin{proof}
By applying Lemma~1 in \cite{bubeck2015finite}
to the processes $Z_{t}$ and $\bar{Z}_{t}$ at time $T=K\eta$, we have
\begin{align}
\Vert X_{T}-\bar{X}_{T}\Vert^{2}
&\leq 
\Vert Z_{T}-\bar{Z}_{T}\Vert^{2}+2\int_{0}^{T}\left\langle Z_{t}-\bar{Z}_{t},-\nu^{J}(X_{t})\right\rangle L(dt)\nonumber
\\
&\qquad\qquad\qquad\qquad\qquad
-2\int_{0}^{T}\left\langle Z_{t}-\bar{Z}_{t},-\bar{\nu}^{J}(\bar{X}_{t})\right\rangle \bar{L}(dt),\label{second:integral}
\end{align}
where the first term on the right hand side in \eqref{second:integral} is zero 
by using the fact that $Z_{T}=\bar{Z}_{T}$ with $T=K\eta$,
and the second integral on the right hand side in \eqref{second:integral}
is zero since the discretized process is constant on the intervals $[k\eta,(k+1)\eta)$
for any $k$ and the local time $\bar{L}$ is a positive combination of Dirac point masses
at $k\eta$, where $k=1,2,\ldots,K$ and furthermore $Z_{k\eta}=\bar{Z}_{k\eta}$ for every $k=1,2,\ldots,K$.
Hence, we obtain
\begin{equation}\label{estimate: local time}
\Vert X_{T}-\bar{X}_{T}\Vert^{2}
\leq 
2\int_{0}^{T}\left\langle Z_{t}-\bar{Z}_{t},-\nu^{J}(X_{t})\right\rangle L(dt).
\end{equation}
Since $\langle x,y\rangle\leq\Vert x\Vert_{\mathcal{C}}h_{\mathcal{C}}(y)$, we have
\begin{equation}
\Vert X_{T}-\bar{X}_{T}\Vert^{2}
\leq 
2\sup_{t\in[0,T]}\Vert Z_{t}-\bar{Z}_{t}\Vert_{\mathcal{C}}
\int_{0}^{T}h_{\mathcal{C}}(-\nu^{J}(X_{t}))L(dt).
\end{equation}
By Cauchy-Schwarz inequality and applying Lemma~\ref{lem:1} and Lemma~\ref{lem:2}, we get
\begin{align}\label{estimate: final}
\mathbb{E}\Vert X_{T}-\bar{X}_{T}\Vert
&\leq 
\sqrt{2}\left(\mathbb{E}\left[\sup_{t\in[0,T]}\Vert Z_{t}-\bar{Z}_{t}\Vert_{\mathcal{C}}\right]\right)^{1/2}
\left(\mathbb{E}\left[\int_{0}^{T}h_{\mathcal{C}}(-\nu^{J}(X_{t}))L(dt)\right]\right)^{1/2}
\nonumber
\\
&\leq\sqrt{2}\left(CM_{\mathcal{C}}\sqrt{2d\eta}\sqrt{\log(T/\eta)}+\frac{\eta(1+\Vert J\Vert_{\infty})L}{r}\right)^{1/2}
\nonumber
\\
&\qquad\qquad\qquad
\cdot\left(R^2 C_{L} +R^2 C_{L}(1+d)T\right)^{1/2}
\nonumber
\\
&\leq\sqrt{2}\left(\sqrt{C}\sqrt{M_{\mathcal{C}}}(2d\eta)^{1/4}(\log(T/\eta))^{1/4}+\frac{\eta(1+\Vert J\Vert_{\infty})L}{r}\right)
\nonumber
\\
&\qquad\qquad\qquad
\cdot\left(R\sqrt{C_{L}}+\sqrt{T}\sqrt{C_{L}}\left(d+1\right)^{1/2}\right),
\end{align}
where we used the inequality that $\sqrt{x+y}\leq\sqrt{x}+\sqrt{y}$
for any $x,y\geq 0$ and for any $x\in\mathcal{C}$, $\Vert x\Vert\leq R$.
This completes the proof.
\end{proof}

By leveraging the above lemma, we can get the following estimate.

\begin{corollary}\label{cor:strengthened}
For any $T=K\eta>0$, we have 
\begin{align}
\sup_{0\le t\le T}\mathbb{E}\Vert X_{t}-\bar{X}_{t}\Vert
&\leq \widetilde{\mathcal{E}}_{\eta}(T)
\label{tilde:E:T:defn},
\end{align}
where
\begin{equation}
\widetilde{\mathcal{E}}_{\eta}(T):=\mathcal{E}_{\eta}(T)+CM_{\mathcal{C}}\sqrt{2d\eta}\sqrt{\log(T/\eta)}+\frac{\eta(1+\Vert J\Vert_{\infty})L}{r},
\end{equation}
where $\mathcal{E}_{\eta}(T)$ is defined in Lemma ~\ref{lem:discretization:final}.
\end{corollary}

\begin{proof}
For any $0\le t \le T$, there exists a maximum $k$ such that $t\in [k\eta, (k+1)\eta)$. Following the estimate \eqref{second:integral} in
    the proof of Lemma \ref{lem:discretization:final}, we have 
    \begin{align}
\Vert X_{t}-\bar{X}_{t}\Vert^{2}
&\leq 
\Vert Z_{t}-\bar{Z}_{t}\Vert^{2}+2\int_{0}^{t}\left\langle Z_{s}-\bar{Z}_{s},-\nu^{J}(X_{s})\right\rangle L(ds)\nonumber
\\
&\qquad\qquad\qquad\qquad\qquad
-2\int_{0}^{t}\left\langle Z_{s}-\bar{Z}_{s},-\bar{\nu}^{J}(\bar{X}_{s})\right\rangle \bar{L}(ds),\label{second:integral-2}
\end{align}
where the second integral on the right hand side of \eqref{second:integral-2} is still zero since the discretized process is constant on the interval $[k\eta, t)$. Applying the inequality $\sqrt{x+y}\le \sqrt{x}+\sqrt{y}$ and estimate \eqref{estimate: local time}, we get 
 \begin{align}
\mathbb E[\Vert X_{t}-\bar{X}_{t}\Vert]
&\leq 
\mathbb E\left[\sup_{0\le s\le t}\Vert Z_{s}-\bar{Z}_{s}\Vert\right]
\nonumber
\\
&\qquad\qquad
+\sqrt{2}\mathbb E\left[\left(\sup_{s\in[0,t]}\Vert Z_{s}-\bar{Z}_{s}\Vert_{\mathcal{C}}
\int_{0}^{t}h_{\mathcal{C}}(-\nu^{J}(X_{s}))L(ds)\right)^{1/2}\right].\nonumber
\end{align}
By combining estimate \eqref{estimate: final} and Lemma \ref{lem:2}, we complete the proof.
\end{proof}

We immediately obtain the following corollary that
bounds the $1$-Wasserstein distance between the distribution of
the continuous-time dynamics SRNLD \eqref{eqn:anti:reflected} and that of $\bar{X}_{K\eta}$.

\begin{corollary}\label{cor:discretization}
For any $K\in\mathbb{N}$,     
\begin{align}
\mathcal{W}_{1}(\mathrm{Law}(X_{K\eta}),\mathrm{Law}(\bar{X}_{K\eta}))
\leq\mathcal{E}_{\eta}(K\eta),
\end{align}
where $\mathcal{E}_{\eta}(\cdot)$ is defined in \eqref{E:T:defn}.
\end{corollary}

\begin{proof}
This is a direct consequence of Lemma~\ref{lem:discretization:final}
and the definition of $1$-Wasserstein distance. 
\end{proof}

Next, we obtain an upper bound on the 1-Wasserstein distance
between the distribution of $\bar{X}_{K\eta}$ and that of $x_{k}$, 
the discrete-time algorithm SRNLMC.

\begin{lemma}\label{lem:additional}
For any $K\in\mathbb{N}$,     
\begin{align}
\mathcal{W}_{1}(\mathrm{Law}(\bar{X}_{K\eta}),\mathrm{Law}(x_{K}))
\leq\sqrt{2}R\sqrt{C_{\infty}}\sqrt{K\eta}\sqrt{\widetilde{\mathcal{E}}_{\eta}(K\eta)},
\end{align}
where $\widetilde{\mathcal{E}}_{\eta}(\cdot)$ is defined in \eqref{tilde:E:T:defn} and 
\begin{equation}
C_{\infty}:=(1+\Vert J\Vert_{\infty})\Vert\nabla f\Vert_{\infty}((1+\Vert J\Vert_{\infty})L+\Vert\nabla f\Vert_{\infty}L_{J}),
\end{equation}
with $\Vert J\Vert_{\infty}:=\Vert J(0)\Vert+L_{J}R$ and $\Vert\nabla f\Vert_{\infty}:=\Vert\nabla f(0)\Vert+LR$.
\end{lemma}

\begin{proof}
The proof is inspired by Lemma~8 in \cite{bubeck2015finite}.
Let $T=K\eta$. 
Given a continuous path $(w_{t})_{t\leq K\eta}$, we define
the map $Q$ by setting $Q((w_{t})_{t\leq K\eta})=y_{K}$, where 
\begin{equation}
y_{k+1}=\mathcal{P}_{\mathcal{C}}^{J}\left(y_{k}+w_{(k+1)\eta}-w_{k\eta}-\eta(I+J(y_{k}))\nabla f(y_{k})\right),\qquad k\leq K-1.
\end{equation}
Then, we have $x_{K}=Q((W_{t})_{t\leq K\eta})$ in distribution 
and $\bar{X}_{K\eta}=Q((\tilde{W}_{t})_{t\leq K\eta})$
with 
\begin{equation}
\tilde{W}_{t}:=W_{t}+\int_{0}^{t}((I+J(\bar{X}_{s}))\nabla f(\bar{X}_{s})-(I+J(X_{s}))\nabla f(X_{s}))ds.
\end{equation}
By the Girsanov formula, one can compute that
\begin{align}
&\mathrm{KL}\left(\mathrm{Law}(\bar{X}_{K\eta})\Vert \mathrm{Law}(x_{k})\right)
\nonumber
\\
&\leq\mathrm{KL}\left(\mathrm{Law}\left((\tilde{W}_{t})_{t\leq K\eta}\right)\Vert\mathrm{Law}\left((W_{t})_{t\leq K\eta}\right)\right)\nonumber
\\
&\leq\frac{1}{2}\mathbb{E}\left[\int_{0}^{K\eta}\left\Vert(I+J(\bar{X}_{t}))\nabla f(\bar{X}_{t})-(I+J(X_{t}))\nabla f(X_{t})\right\Vert^{2}dt\right]
\nonumber
\\
&\leq\sup_{x\in\mathcal{C}}\Vert(I+J(x))\nabla f(x)\Vert
\mathbb{E}\left[\int_{0}^{K\eta}\left\Vert(I+J(\bar{X}_{t}))\nabla f(\bar{X}_{t})-(I+J(X_{t}))\nabla f(X_{t})\right\Vert dt\right].
\end{align}
Note that Assumption~\ref{assump:domain} and Assumption~\ref{assump:f:J}
imply that for any $x\in\mathcal{C}$, $\Vert J(x)\Vert\leq\Vert J(0)\Vert+L_{J}\Vert x\Vert\leq\Vert J\Vert_{\infty}:=\Vert J(0)\Vert+L_{J}R$. Similarly, for any $x\in\mathcal{C}$, $\Vert\nabla f(x)\Vert\leq\Vert\nabla f\Vert_{\infty}:=\Vert\nabla f(0)\Vert+LR$.
This implies that 
\begin{equation}
\sup_{x\in\mathcal{C}}\Vert(I+J(x))\nabla f(x)\Vert
\leq(1+\Vert J\Vert_{\infty})\Vert\nabla f\Vert_{\infty},
\end{equation}
and $(I+J)\nabla f$ is $((1+\Vert J\Vert_{\infty})L+\Vert\nabla f\Vert_{\infty}L_{J})$-Lipschitz.
Thus, 
\begin{equation}
\mathrm{KL}\left(\mathrm{Law}(\bar{X}_{K\eta})\Vert \mathrm{Law}(x_{k})\right)
\leq C_{\infty}\mathbb{E}\left[\int_{0}^{K\eta}\left\Vert\bar{X}_{t}-X_{t}\right\Vert dt\right],
\end{equation}
where $C_{\infty}:=(1+\Vert J\Vert_{\infty})\Vert\nabla f\Vert_{\infty}((1+\Vert J\Vert_{\infty})L+\Vert\nabla f\Vert_{\infty}L_{J})$. By applying Corollary~\ref{cor:strengthened}, we obtain
\begin{equation}
\mathrm{KL}\left(\mathrm{Law}(\bar{X}_{K\eta})\Vert \mathrm{Law}(x_{k})\right)
\leq C_{\infty}\cdot(K\eta)\cdot\widetilde{\mathcal{E}}_{\eta}(K\eta).
\end{equation}
By Pinsker's inequality and the fact $1$-Wasserstein distance can be bounded by TV distance on a bounded domain \cite{GibbsSu2002}, we obtain
\begin{equation}
\mathcal{W}_{1}(\mathrm{Law}(\bar{X}_{K\eta}),\mathrm{Law}(x_{K}))
\leq 2R\cdot\mathrm{TV}(\mathrm{Law}(\bar{X}_{K\eta}),\mathrm{Law}(x_{K}))
\leq\sqrt{2}R\sqrt{C_{\infty}}\sqrt{K\eta}\sqrt{\widetilde{\mathcal{E}}_{\eta}(K\eta)}.
\end{equation}
This completes the proof.
\end{proof}

By combining Theorem~\ref{thm:TV}, Corollary~\ref{cor:discretization} and Lemma~\ref{lem:additional}, 
we obtain the non-asymptotic convergence guarantees for discrete-time algorithm SRNLMC
to the target distribution in $1$-Wasserstein distance.

\begin{theorem}\label{thm:final}
For any $K\in\mathbb{N}$, 
\begin{align}
\mathcal{W}_{1}(\mathrm{Law}(x_{K}),\pi)
\leq 2R\mathcal{K}e^{-\rho_{J}K\eta}
+\mathcal{E}_{\eta}(K\eta)
+\sqrt{2}R\sqrt{C_{\infty}}\sqrt{K\eta}\sqrt{\widetilde{\mathcal{E}}_{\eta}(K\eta)},
\end{align}  
where $\mathcal{E}_{\eta}(\cdot)$ and $\widetilde{\mathcal{E}}_{\eta}(\cdot)$ are defined in \eqref{E:T:defn} and \eqref{tilde:E:T:defn}.
\end{theorem}

\begin{proof}
It is a direct consequence of the triangle inequality
\begin{align}
\mathcal{W}_{1}(\mathrm{Law}(x_{K}),\pi)
&\leq\mathcal{W}_{1}(\mathrm{Law}(x_{K}),\mathrm{Law}(\bar{X}_{K\eta}))
\nonumber
\\
&\qquad\qquad+\mathcal{W}_{1}(\mathrm{Law}(\bar{X}_{K\eta}),\mathrm{Law}(X_{K\eta}))+\mathcal{W}_{1}(\mathrm{Law}(X_{K\eta}),\pi),
\end{align} 
and by applying Theorem~\ref{thm:TV}, Corollary~\ref{cor:discretization} and Lemma~\ref{lem:additional}.
\end{proof}

Based on Theorem~\ref{thm:final}, we obtain the iteration complexity
in terms of the accuracy level $\varepsilon$, the dimension $d$ and the spectral gap $\rho_{J}$
in the following corollary.

\begin{corollary}\label{cor:final}
For any given accuracy level $\varepsilon>0$, 
$\mathcal{W}_{1}(\mathrm{Law}(x_{K}),\pi)\leq\tilde{\mathcal{O}}(\varepsilon)$ provided that
\begin{equation}
K\eta=\frac{\log(1/\varepsilon)}{\rho_{J}},
\end{equation}
and
\begin{equation}
\eta\leq\frac{\rho_{J}^{6}\varepsilon^{8}}{\log(1/\rho_{J})d^{10}\log(d)\log(1/\varepsilon)},
\end{equation}
where $\tilde{\mathcal{O}}$ hides the dependence on $\log\log(1/\varepsilon)$, $\log\log(d)$ and $\log\log(1/\rho_{J})$.
\end{corollary}

\begin{proof}
It follows from Theorem~\ref{thm:final} 
and the definitions of $\mathcal{E}_{\eta}(\cdot)$ and $\widetilde{\mathcal{E}}_{\eta}(\cdot)$ in \eqref{E:T:defn} and \eqref{tilde:E:T:defn}
that
\begin{align}
&\mathcal{W}_{1}(\mathrm{Law}(x_{K}),\pi)
\nonumber
\\
&\leq 2R\mathcal{K}e^{-\rho_{J}K\eta}
+\sqrt{2}\left(\sqrt{C}\sqrt{M_{\mathcal{C}}}(2d\eta)^{1/4}(\log(K))^{1/4}+\frac{\eta(1+\Vert J\Vert_{\infty})L}{r}\right)
\nonumber
\\
&\qquad\qquad\qquad\qquad\qquad
\cdot\left(R\sqrt{C_{L}}+\sqrt{K\eta}\sqrt{C_{L}}\left(d+1\right)^{1/2}\right)
\nonumber
\\
&\qquad
+\sqrt{2}R\sqrt{C_{\infty}}\sqrt{K\eta}\Bigg(\sqrt{2}\left(\sqrt{C}\sqrt{M_{\mathcal{C}}}(2d\eta)^{1/4}(\log(K))^{1/4}+\frac{\eta(1+\Vert J\Vert_{\infty})L}{r}\right)
\nonumber
\\
&
\cdot\left(R\sqrt{C_{L}}+\sqrt{K\eta}\sqrt{C_{L}}\left(d+1\right)^{1/2}\right)
+CM_{\mathcal{C}}\sqrt{2d\eta}\sqrt{\log(K)}+\frac{\eta(1+\Vert J\Vert_{\infty})L}{r}\Bigg)^{1/2},
\end{align}  
where $C>0$ is a universal constant.
Therefore, we have
\begin{equation}
\mathcal{W}_{1}(\mathrm{Law}(x_{k}),\pi)
\leq\mathcal{O}\left(e^{-\rho_{J}K\eta}\right)
+\mathcal{O}\left(\eta^{1/8}(\log(K))^{1/4}(K\eta)^{3/4}d^{5/4}\right).
\end{equation}
Therefore, for any given accuracy level $\varepsilon>0$, 
$\mathcal{O}\left(e^{-\rho_{J}K\eta}\right)\leq\mathcal{O}(\varepsilon)$
provided that
\begin{equation}
K\eta=\frac{\log(1/\varepsilon)}{\rho_{J}},
\end{equation}
and
\begin{align}
\mathcal{O}\left(\eta^{1/8}(\log(K))^{1/4}(K\eta)^{3/4}d^{5/4}\right)
&\leq\mathcal{O}\left(d^{5/4}\eta^{1/8}\frac{\left(\log(1/\varepsilon)\right)^{3/4}}{\rho_{J}^{3/4}}\left(\log\left(\frac{\log(1/\varepsilon)}{\eta\rho_{J}}\right)\right)^{1/4}\right)
\nonumber
\\
&\leq\tilde{\mathcal{O}}(\varepsilon),
\end{align}
given that
\begin{equation}
\eta\leq\frac{\rho_{J}^{6}\varepsilon^{8}}{\log(1/\rho_{J})d^{10}\log(d)\log(1/\varepsilon)},
\end{equation}
where $\tilde{\mathcal{O}}$ hides the dependence on $\log\log(1/\varepsilon)$, $\log\log(d)$ and $\log\log(1/\rho_{J})$.
\end{proof}

\begin{remark}
Corollary~\ref{cor:final} implies that
for any given accuracy level $\varepsilon>0$, 
$\mathcal{W}_{1}(\mathrm{Law}(x_{k}),\pi)\leq\tilde{\mathcal{O}}(\varepsilon)$ by taking 
\begin{equation}
\eta=\frac{\rho_{J}^{6}\varepsilon^{8}}{\log(1/\rho_{J})d^{10}\log(d)\log(1/\varepsilon)},
\end{equation}
and
\begin{equation}
K=\frac{\log(1/\varepsilon)}{\eta\rho_{J}}
=\frac{\log(1/\rho_{J})d^{10}\log(d)\left(\log(1/\varepsilon)\right)^{2}}{\rho_{J}^{7}\varepsilon^{8}}
=\tilde{\mathcal{O}}\left(\frac{\log(1/\rho_{J})d^{10}}{\rho_{J}^{7}\varepsilon^{8}}\right),
\end{equation}
which provides the iteration complexity of the algorithm, 
where $\mathcal{O}$ hides the logarithmic dependence on $d$, $\varepsilon$.
Since $\rho_{J}\geq\rho_{0}>0$ (see Theorem~\ref{thm:TV}), SRNLMC achieves a lower complexity 
compared to PLMC ($J=0$)
for the given accuracy level $\varepsilon$, and hence 
we showed the by breaking reversibility, in the context of constrained sampling, 
acceleration is achievable.
\end{remark}

%%%%%%%%%%%%%%%%%%%%%%%%%%%%%%%%
\section{Numerical Experiments}\label{sec:numerical}

In Section~\ref{subsec:toy:example}, we will first compare the \emph{iteration complexity} of skew-reflected non-reversible Langevin Monte Carlo (SRNLMC) to the one of projected Langevin Monte Carlo (PLMC) in $1$-Wasserstein distance by using synthetic data in a toy example of truncated standard multivariate normal distribution. 

By introducing \emph{stochastic gradients}, we will propose \emph{skew-reflected non-reversible stochastic gradient Langevin dynamics} (SRNSGLD) and \emph{projected stochastic gradient Langevin dynamics} (PSGLD). In Section~\ref{subsec:Bayesian:linear:example}, we will compare SRNSGLD to PSGLD in terms of \emph{mean squared error} by using synthetic data in the example of constrained Bayesian linear regression. Furthermore, we will consider the example of constrained Bayesian logistic regression and we compare the accuracy, which is defined as the ratio of the correctly predicted labels over the whole dataset in deep learning experiments, of the algorithms by using either synthetic or real data in Section~\ref{subsec:Bayesian:log:example}.

%%%%%%%%%%%%%%%%%%%%%
%Let us start with a toy example of sampling the truncated standard multivariate normal distribution.

\subsection{Toy Example: Truncated Standard Multivariate Normal Distribution}
\label{subsec:toy:example}
Considering the truncated standard multivariate normal distribution in certain convex sets $\mathcal{C}\in \mathbb{R}^3$ with density function: 
\begin{equation}
\pi(x) = \frac{1}{Z}\exp\left\{-\frac{\|x\|^2}{2}\right\} \mathbf{1}_{\mathcal{C}},
\end{equation}
where $\mathbf{1}_{\mathcal{C}}$ denotes the indicator function on $\mathcal{C}$ and $Z:=\int_{\mathcal{C}}e^{-\frac{\|x\|^2}{2}}dx$
is a normalizing constant. In our experiments, we take the convex constraint sets to be ball $\mathcal{C}_1$ and cubic $\mathcal{C}_2$   in $\mathbb{R}^3$: 
\begin{equation}
\label{eqn:ball}
 \mathcal{C}_1 = \{x \in \mathbb{R}^3:\|x\|^2_2 \leq 1\}, \qquad \mathcal{C}_2 = [-1,1]^3.
\end{equation}
For these $2$ choices of convex sets, we can derive their skew-projection formula explicitly by computing out the skew unit normal vector direction using geometry. We set the skewed matrix $J=J_{a}$ in SRNLMC to be state-independent, 
that is defined as:
\begin{equation}
\label{eqn:skew:matrix}
J_a = \begin{bmatrix}
0 & a & 0 \\
-a & 0 & a \\
0 & -a & 0
\end{bmatrix},
\end{equation}
where $a \neq 0$ is a self-chosen constant parameter.
Then we approximate Gibbs distribution by the rejection sampling approach, i.e. sampling from the 3-dimensional standard normal distribution and discarding the sample points that are out of the constraint set until obtaining the required number of sample points. Finally, we compare the $1$-Wasserstein distances to the approximate distribution of SRNLMC and PLMC in each dimension. 

Let the constraint set be the ball $\mathcal{C}_1$ defined in~\eqref{eqn:ball}. We chose the skewed matrix $J_1$ with $a = 1$ in~\eqref{eqn:skew:matrix}, and simulated $n=3000$ samples and took $5000$ iterates starting from the initial point $x_0 = [0.3,0.6,-0.4]^{\top}$ with the stepsize $\eta=10^{-4}$. In Figure~\ref{visualized_ball}, we showed the first $2$ dimensions of the target distribution in the left panel, and the first $2$ dimensions of the sample distributions by SRNLMC and PLMC in the middle and right panels. These results showed that both SRNLMC and PLMC can successfully converge to the truncated standard normal distribution in our setting. 
\begin{figure}[htbp]
    \centering
    \begin{subfigure}{0.3\textwidth}
        \centering
        \includegraphics[width=\linewidth]{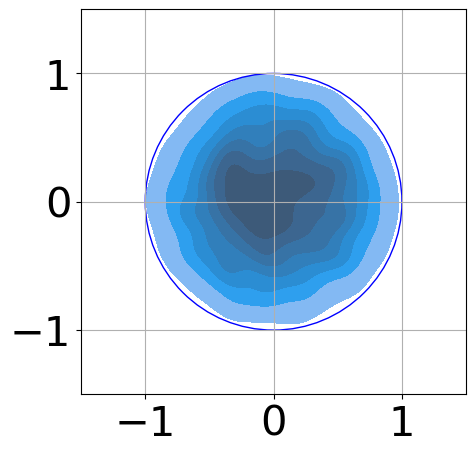}
        \caption{Target distribution}
    \end{subfigure}%
    \hspace{0.5cm} 
    \begin{subfigure}{0.3\textwidth}
        \centering
        \includegraphics[width=\linewidth]{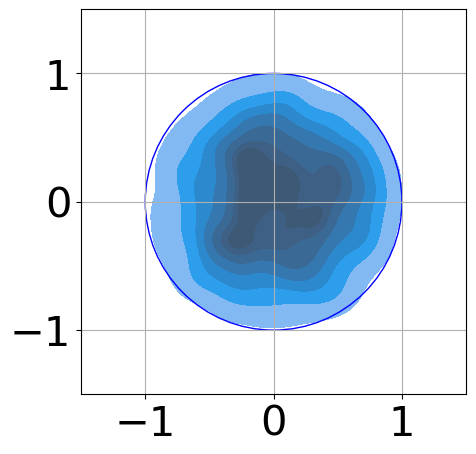}
        \caption{SRNLMC}
    \end{subfigure}%
    \hspace{0.5cm} 
    \begin{subfigure}{0.3\textwidth}
        \centering
        \includegraphics[width=\linewidth]{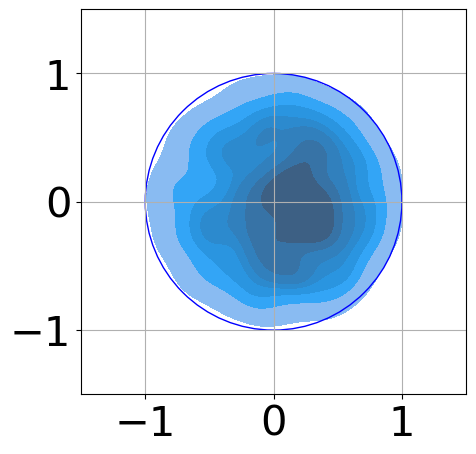}
        \caption{PLMC}
    \end{subfigure}
    \caption{Visualized density plots for the first 2 dimensions in ball constraint}
    \label{visualized_ball}
\end{figure}
Then we compared the convergence of SRNLMC to PLMC by computing the $1$-Wasserstein distance in each dimension. By taking $5000$ iterations with the ball constraint $\mathcal{C}_1$ defined in~\eqref{eqn:ball} for both PLMC and SRNLMC, we obtained the plots in Figure~\ref{1wasserball} where the blue line represents SRNLMC and the orange line represents PLMC in each subfigure.
\begin{figure}[htbp]
    \centering
    \begin{subfigure}{0.3\textwidth}
        \centering
        \includegraphics[width=\linewidth]{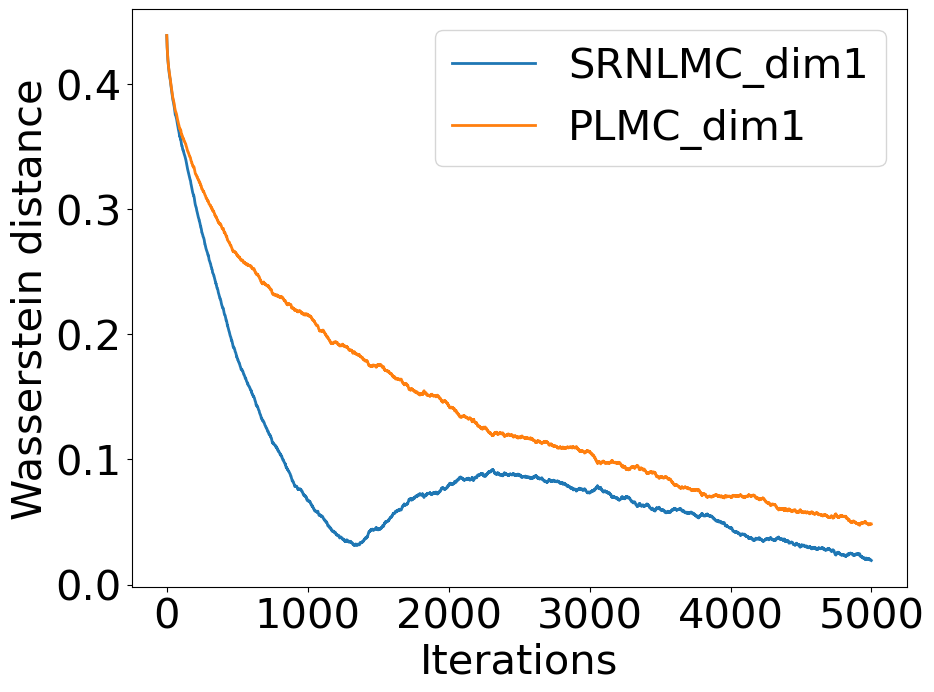}
        \caption{Dimension 1}
    \end{subfigure}%
    \hspace{0.5cm}
    \begin{subfigure}{0.3\textwidth}
        \centering
        \includegraphics[width=\linewidth]{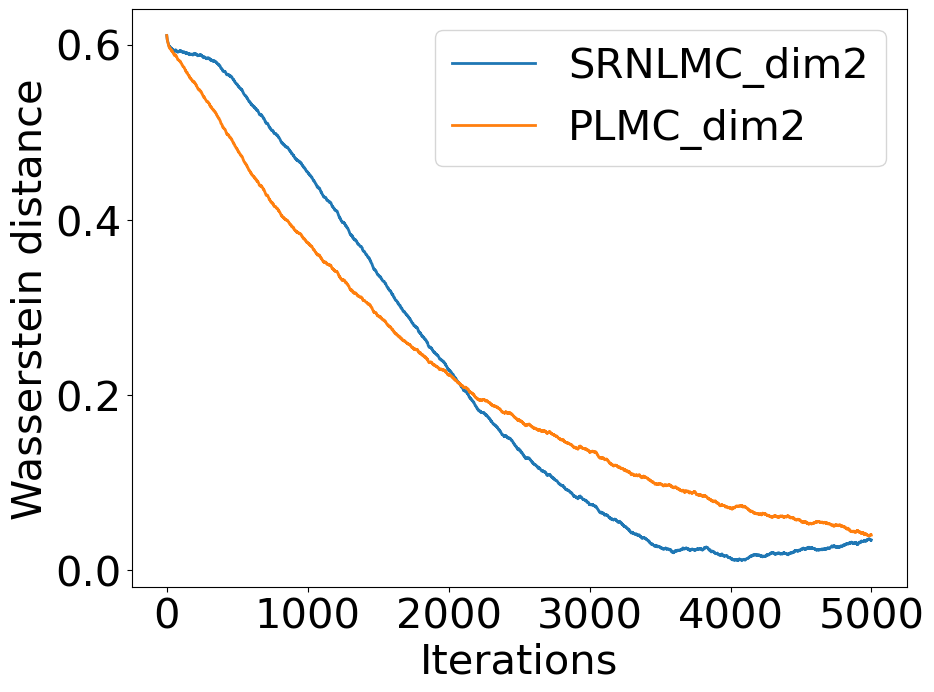}
        \caption{Dimension 2}
    \end{subfigure}%
    \hspace{0.5cm}
    \begin{subfigure}{0.3\textwidth}
        \centering
        \includegraphics[width=\linewidth]{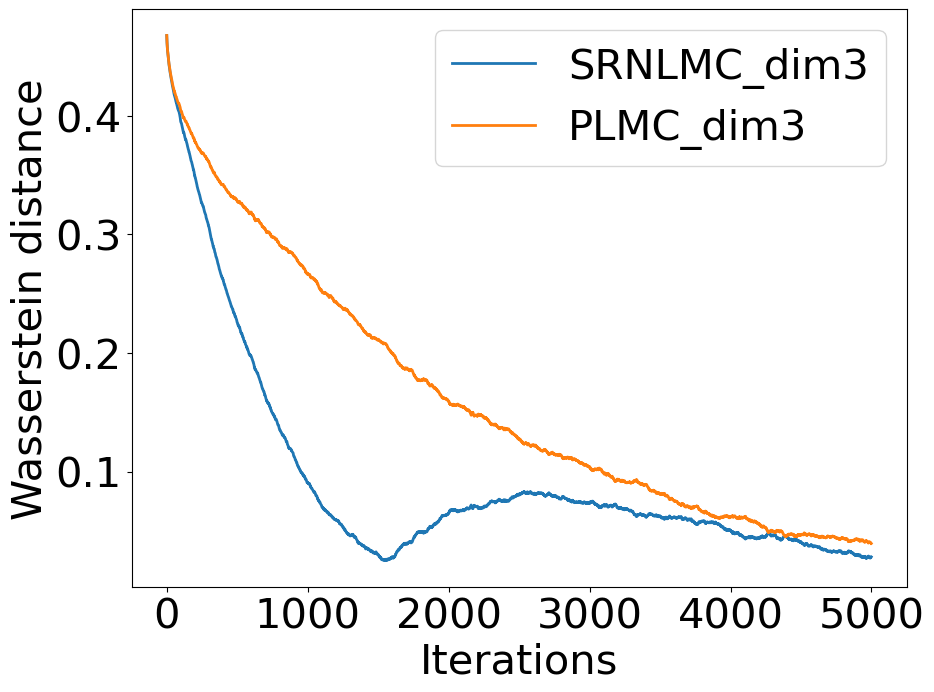}
        \caption{Dimension 3}
    \end{subfigure}
    \caption{$1$-Wasserstein distance in each dimension of PLMC and SRNLMC in ball constraint}
    \label{1wasserball}
\end{figure}
In the same line, we considered the constraint set to be the cubic $\mathcal{C}_2$ defined in~\eqref{eqn:ball}, and took $J_2$ to be the skewed matrix with $a = 2$ defined in~\eqref{eqn:skew:matrix}, and the initial point to be $x_0 = [0.5,-0.2,0.8]^{\top}$. The visualized results are shown in Figure~\ref{visualized_cubic} and the corresponding convergence results in Figure~\ref{1wassercubic}.
\begin{figure}[htbp]
    \centering
    \begin{subfigure}{0.3\textwidth}
        \centering
        \includegraphics[width=\linewidth]{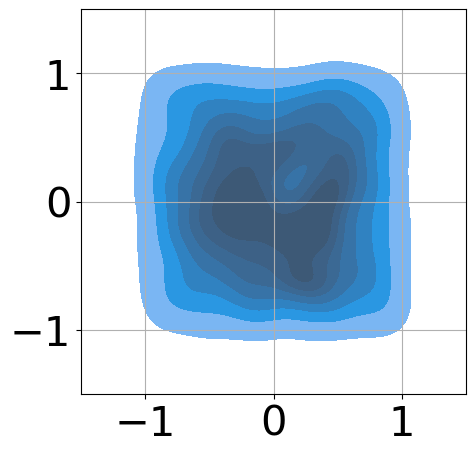}
        \caption{Target distribution}
    \end{subfigure}%
    \hspace{0.5cm}
    \begin{subfigure}{0.3\textwidth}
        \centering
        \includegraphics[width=\linewidth]{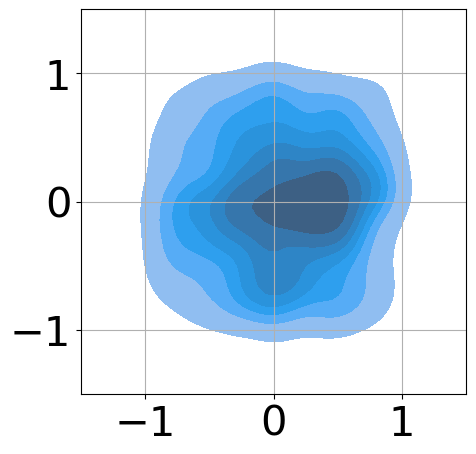}
        \caption{SRNLMC}
    \end{subfigure}%
    \hspace{0.5cm}
    \begin{subfigure}{0.3\textwidth}
        \centering
        \includegraphics[width=\linewidth]{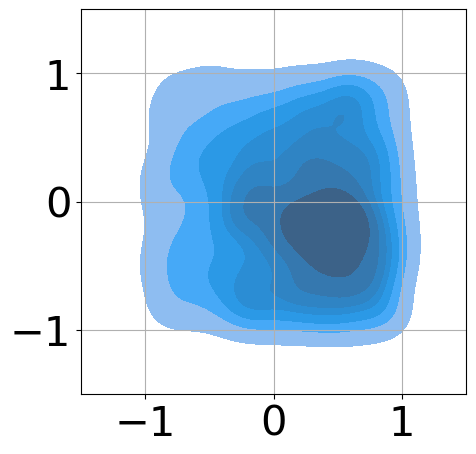}
        \caption{PLMC}
    \end{subfigure}
    \caption{Visualized density plots for the first 2 dimensions in cubic constraint}
    \label{visualized_cubic}
\end{figure}
\begin{figure}[htbp]
    \centering
    \begin{subfigure}{0.3\textwidth}
        \centering
        \includegraphics[width=\linewidth]{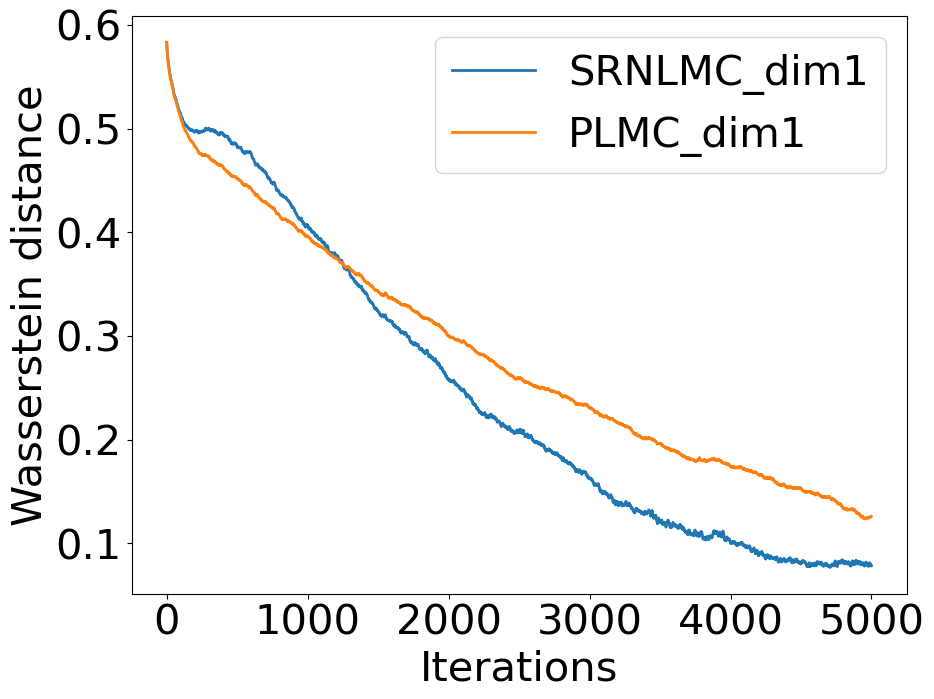}
        \caption{Dimension 1}
    \end{subfigure}%
    \hspace{0.5cm}
    \begin{subfigure}{0.3\textwidth}
        \centering
        \includegraphics[width=\linewidth]{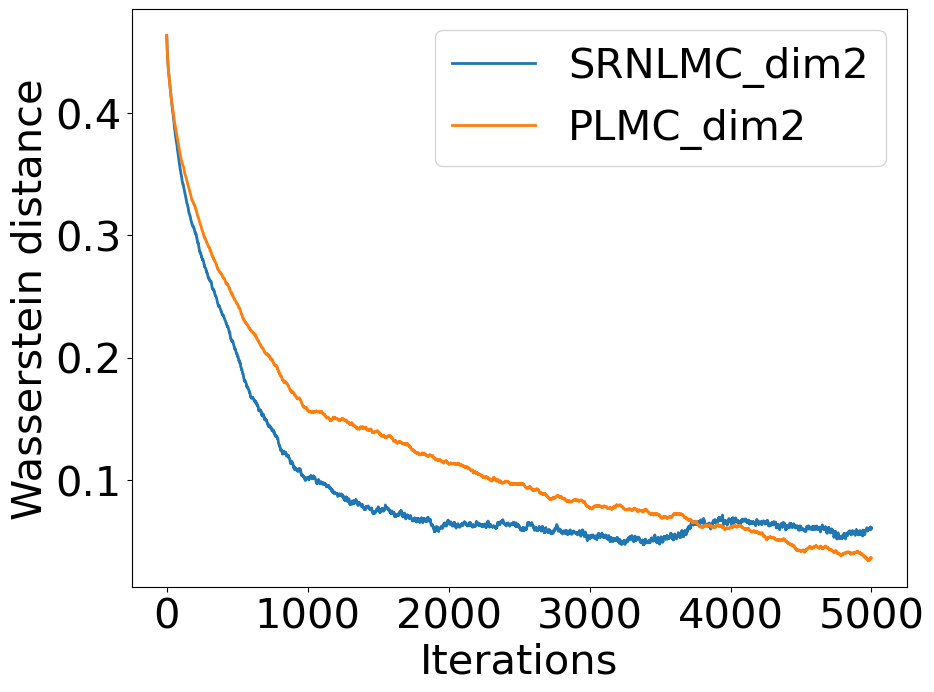}
        \caption{Dimension 2}
    \end{subfigure}%
    \hspace{0.5cm}
    \begin{subfigure}{0.3\textwidth}
        \centering
        \includegraphics[width=\linewidth]{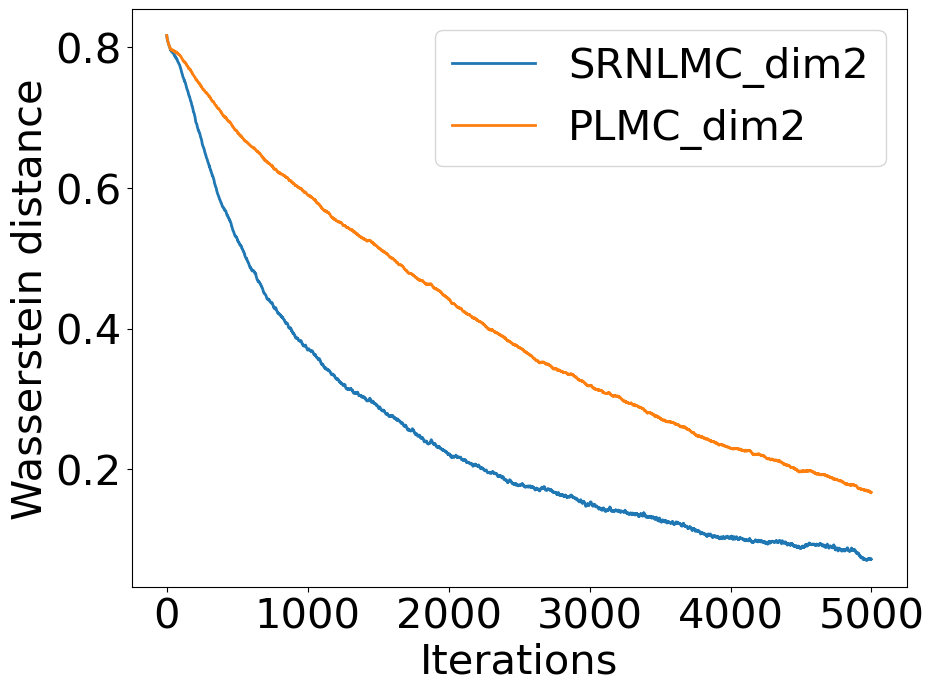}
        \caption{Dimension 3}
    \end{subfigure}
    \caption{$1$-Wasserstein distance in each dimension of PLMC and SRNLMC in cubic constraint}
    \label{1wassercubic}
\end{figure}

We can observe from Figure~\ref{1wasserball} and Figure~\ref{1wassercubic} that with an appropriate chosen constant $a \neq 0$ in skewed matrix $J_a$ defined in~\eqref{eqn:skew:matrix}, the iteration complexity of SRNLMC is lower than the one of PLMC to achieve the same order of $1$-Wasserstein distance between the sample and target distributions. This experiment result is consistent with our theoretical result in Theorem~\ref{thm:final} that SRNLMC can outperform PLMC, i.e. the skewed matrix $J$ can help achieve acceleration in the constraint sampling problem.

%%%%%%%%%%%%%%%%%%%%%

\subsection{Constrained Bayesian Linear Regression}
\label{subsec:Bayesian:linear:example}

In our next set of experiments, we test our algorithms in the constrained Bayesian linear regression models. In this experiment, we consider the case when the constraint set is a 2-dimensional centered unit disk such that
\begin{equation}
\label{eqn:disk}
\mathcal{C}=\left\{x \in \mathbb{R}^2 :\|x\|_2^2 \leq 1\right\},
\end{equation}
which corresponds to the ridge Bayesian linear regression. %Ridge Bayesian linear regression is particularly effective in scenarios with multicollinearity or limited data and provides a principled way to incorporate prior knowledge and manage overfitting in regression models. 
We consider the linear regression model: 
\begin{equation}
\delta_j \sim \mathcal{N}(0,0.25), \quad a_j \sim \mathcal{N}(0, I), \quad y_j=x_{\star}^{\top} a_j+\delta_j, \quad x_{\star}=[1,1]^{\top}.  
\end{equation} 
The prior distribution is a uniform distribution, in which case the constraints are satisfied. Our goal is to generate the posterior distribution given by
\begin{equation}
\pi(x) \propto \exp\left\{-\frac{1}{2}\sum_{j=1}^n\left(y_j-x^{\top} a_j\right)^2\right\}\mathbf{1}_{\mathcal{C}},
\end{equation}
where $\mathbf{1}_{\mathcal{C}}$ is the indicator function for the constraint set $\mathcal{C}$ defined in~\eqref{eqn:disk} and $n$ is the total number of data points in the training set. 
%In order to check the convergence performance of SRNSGLD visually, we started with a synthetic $2$-dimensional problem. We consider the model:
%\begin{equation}
%\delta_j \sim \mathcal{N}(0,0.25), \quad a_j \sim \mathcal{N}(0, I), \quad y_j=x_{\star}^{\top} a_j+\delta_j, \quad x_{\star}=[1,1]^{\top},  
%\end{equation}

For this experiment, to further illustrate the efficiency
of our algorithms, we will introduce a stochastic gradient setting,
and we will show that our algorithms work well in the presence
of stochastic gradients.

In Section~\ref{sec:intro}, we proposed \emph{skew-reflected non-reversible Langevin Monte Carlo} (SRNLMC) algorithm \eqref{eqn:algorithm}.
In practice, one often uses stochastic gradient, and we can therefore
also consider the \emph{skew-reflected non-reversible stochastic gradient Langevin dynamics} (SRNSGLD):
%\begin{equation}\label{eqn:SRNSGLD}
%x_{k+1}=\mathcal{R}^J_{\mathcal{C}}\left(x_{k}-\eta(I+J(x_{k}))\nabla f(x_{k},\Omega_{k+1})+\sqrt{2\eta}\xi_{k+1}\right),
%\end{equation}
\begin{equation}\label{eqn:SRNSGLD}
x_{k+1}=\mathcal{P}^J_{\mathcal{C}}\left(x_{k}-\eta(I+J(x_{k}))\nabla f(x_{k},\Omega_{k+1})+\sqrt{2\eta}\xi_{k+1}\right),
\end{equation}
where $\mathcal{P}^J_{\mathcal{C}}$ is the skew-projection onto $\mathcal{C}$, $\xi_{k}$ are i.i.d. Gaussian random vectors $\mathcal{N}(0,I)$ and $\nabla f(x_{k},\Omega_{k+1})$ is a conditionally unbiased estimator of $\nabla f(x_{k})$.
For example, a common choice is the mini-batch setting, in which the full gradient is of the form 
$\nabla f(x)=\frac{1}{n}\sum_{i=1}^{n}\nabla f_{i}(x)$, 
where $n$ is the number of data points and $f_{i}(x)$
associates with the $i$-th data point, 
and the stochastic gradient is given by 
$\nabla f(x_{k},\Omega_{k+1}):=\frac{1}{b}\sum\nolimits_{i\in\Omega_{k+1}}\nabla f_{i}(x_{k})$
with $b\ll n$ being the batch-size
and $\Omega_{k+1}^{(j)}$ are uniformly
sampled random subsets of $\{1,2,\ldots,n\}$
with $|\Omega_{k+1}|=b$ and i.i.d. over $k$. 
If $J \equiv 0$, we obtain \emph{projected stochastic gradient Langevin dynamics~(PSGLD).}

Under the stochastic gradient setting, we chose the batch size $b=50$ and generated $n=10000$ data points $\left(a_j, y_j\right)$ by SRNSGLD, where we took the skewed matrix $J_2$ with $a=2$ defined by~\eqref{eqn:skew:matrix}. By taking $600$ iterations with the stepsize $\eta = 10^{-4}$, we get Figure~\ref{linear} and Figure~\ref{mse}. In Figure~\ref{linear}, we present the prior uniform distribution in the left panel. In the middle and right panels, our simulation shows that both SRNSGLD and PSGLD can converge to the samples whose distribution concentrates around the closest position to the target value, with the red stars shown outside the constraints.
In Figure~\ref{mse}, the blue line presents the mean squared error of SRNSGLD in 600 iterations and the that of PSGLD, where the mean squared error in the $k$-th run is computed by the formula $\operatorname{MSE}_k:=\frac{1}{n} \sum_{j=1}^n\left(y_j-\left(x_k\right)^{\top} a_j\right)^2$. Although we did not provide theoretical guarantees for the acceleration on the MSE, through our experiments, SRNSGLD demonstrates a better convergence performance.
\begin{figure}[htbp]
    \centering
    \begin{subfigure}{0.3\textwidth}
        \centering
        \includegraphics[width=\linewidth]{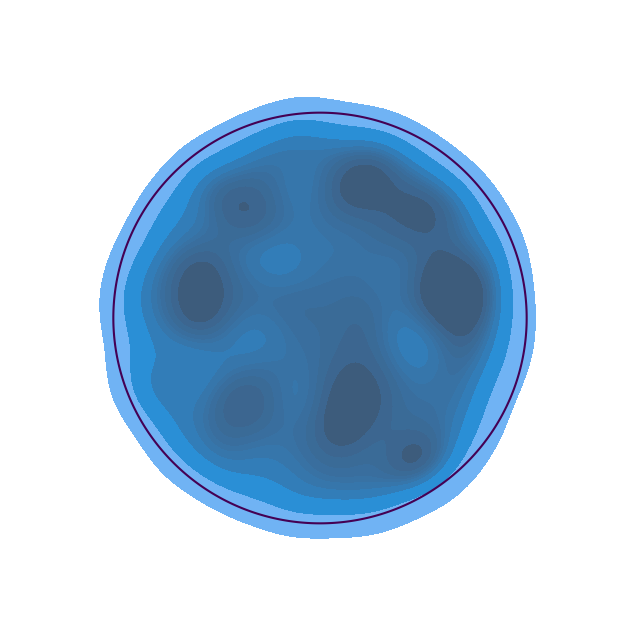}
        \caption{Prior}
    \end{subfigure}%
    \hspace{0.5cm}
    \begin{subfigure}{0.3\textwidth}
        \centering
        \includegraphics[width=\linewidth]{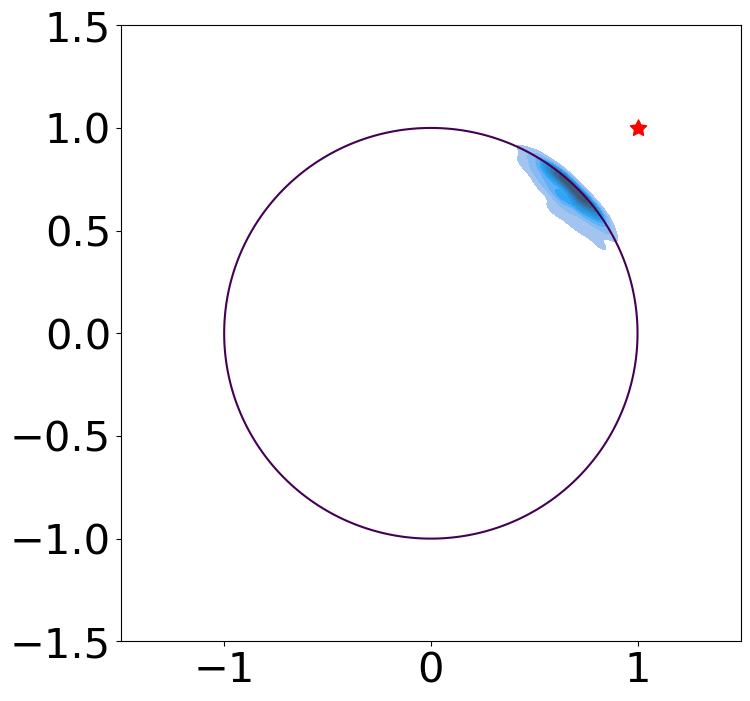}
        \caption{SRNSGLD}
    \end{subfigure}%
    \hspace{0.5cm}
    \begin{subfigure}{0.3\textwidth}
        \centering
        \includegraphics[width=\linewidth]{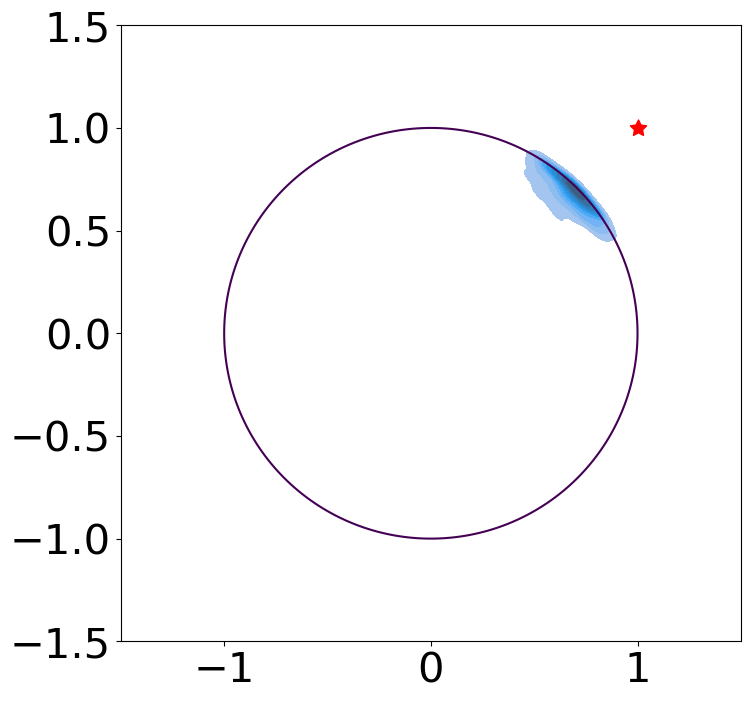}
        \caption{PSGLD}
    \end{subfigure}%
    \caption{Prior and posterior distributions plot with disk constraint}
    \label{linear}
\end{figure}

\begin{figure}[htbp]
        \centering
        \includegraphics[width=0.5\textwidth]{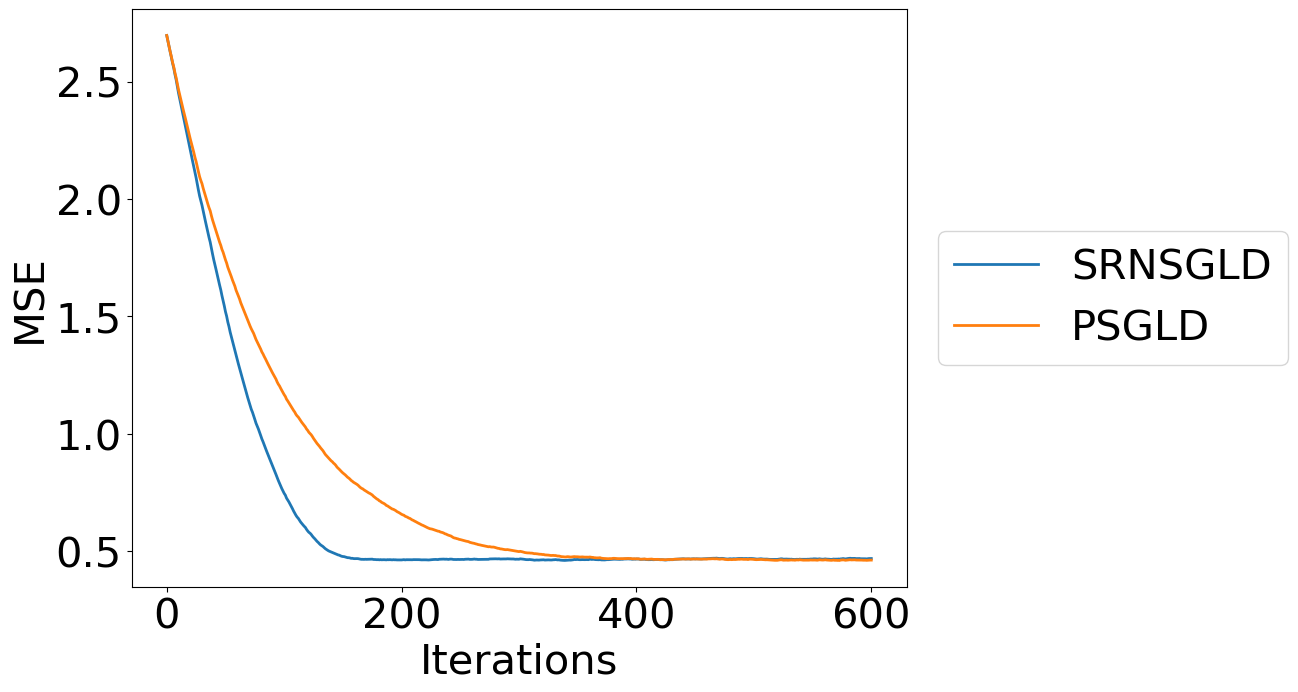}
        \caption{MSE result of SRNSGLD and PSGLD for the constrained Bayesian linear regression}
        \label{mse}
\end{figure}

%%%%%%%%%%%%%%%

\subsection{Constrained Bayesian Logistic Regression}
\label{subsec:Bayesian:log:example}

To test the performance of our algorithm in binary classification problems, we implement the constrained Bayesian linear regression models on both synthetic data and real data, where the constraint is the centered unit ball $\mathcal{C}$ in $\mathbb{R}^d$ such that
\begin{equation}
\mathcal{C}=\left\{x \in \mathbb{R}^d :\|x\|_2^2 \leq 1\right\}.
\end{equation}

Suppose we can access a dataset $Z=\left\{z_j\right\}_{j=1}^{n}$ where $z_j=\left(X_j, y_j\right), X_j \in \mathbb{R}^d$ are the features and $y_j \in\{0,1\}$ are the labels with the assumption that $X_j$ are independent and the probability distribution of $y_j$ given $X_j$ and the regression coefficients $\beta \in \mathbb{R}^d$ are given by
\begin{equation}
\mathbb{P}\left(y_j=1 \mid X_j, \beta\right)=\frac{1}{1+e^{-\beta^{\top} X_j}}.
\end{equation}
In all our experiments, we choose the prior distribution of $\beta$ to be the uniform distribution in the ball constraint. Then the goal of the constrained Bayesian logistic regression is to sample from $\pi(\beta) \propto e^{-f(\beta)}\mathbf{1}_\mathcal{C}$ with $f(\beta)$:
\begin{equation}
f(\beta):=-\sum_{j=1}^{n} \log p\left(y_j \mid X_j, \beta\right)-\log p(\beta)=\sum_{j=1}^{n} \log \left(1+e^{-\beta^{\top} X_j}\right)+\log V_d,
\end{equation}
where $V_d:=\frac{\pi^{d/2}}{\Gamma(\frac{d}{2}+1)}$ is the volume of the unit ball in $\mathbb{R}^d$.

%%%%%%%%%%%%%%%%%%%%%%%%%%%%%%%%%%%
\subsubsection{Synthetic data}

Consider the following example of $d = 3$. 
First, we generated $n=2000$ synthetic data by the following model
\begin{equation}
X_j \sim \mathcal{N}\left(0,2 I_3\right), \quad p_j \sim \mathcal{U}(0,1), \quad y_j= \begin{cases}1 & \text { if } p_j \leq \frac{1}{1+e^{-\beta^{\top} X_j}} \\ 0 & \text { otherwise }\end{cases},
\end{equation}
where $\mathcal{U}(0,1)$ is the uniform distribution on $[0,1]$ and the prior distribution of $ \beta=\left[\beta_1, \beta_2, \beta_3\right]^{\top} \in \mathbb{R}^3$ is a uniform distribution on the centered unit ball in $\mathbb{R}^3$. Then we implemented them with $1000$ iterations by choosing the stepsize $\eta=10^{-4}$ and the batch size $b=50$. For the constrained Bayesian logistic regression, it is not practical for us to compute the $1$-Wasserstein distance between the approximated Gibbs distribution and the empirical distribution. Hence, for such a binary classification problem, we can use the accuracy over training set and test set to measure the goodness of the convergence performance. 

In this experiment, we chose the skewed matrix $J_2$ in~\eqref{eqn:skew:matrix} with $a = 2$ for SRNSGLD and we used $20\%$ of the whole dataset as the test set. The mean and standard deviation of the accuracy distribution are shown in Figure~\ref{logsyth}. We observe that to achieve the same level of accuracy among the test set, the SRNSGLD needs fewer iterations than the PSGLD.
\begin{figure}[htbp]
    \centering
    \begin{subfigure}{0.35\textwidth}
        \centering
        \includegraphics[width=\linewidth]{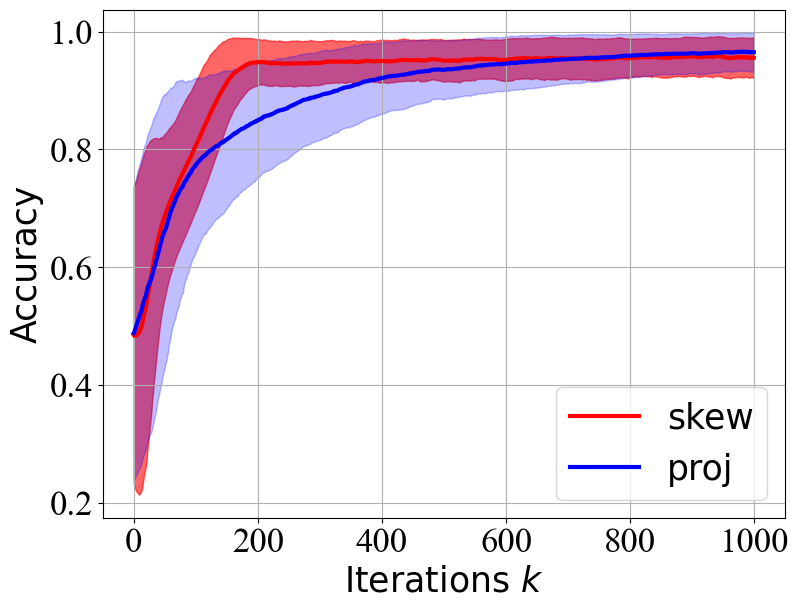}
        \caption{Accuracy over the training set}
    \end{subfigure}%
    \hspace{0.8cm}
    \begin{subfigure}{0.35\textwidth}
        \centering
        \includegraphics[width=\linewidth]{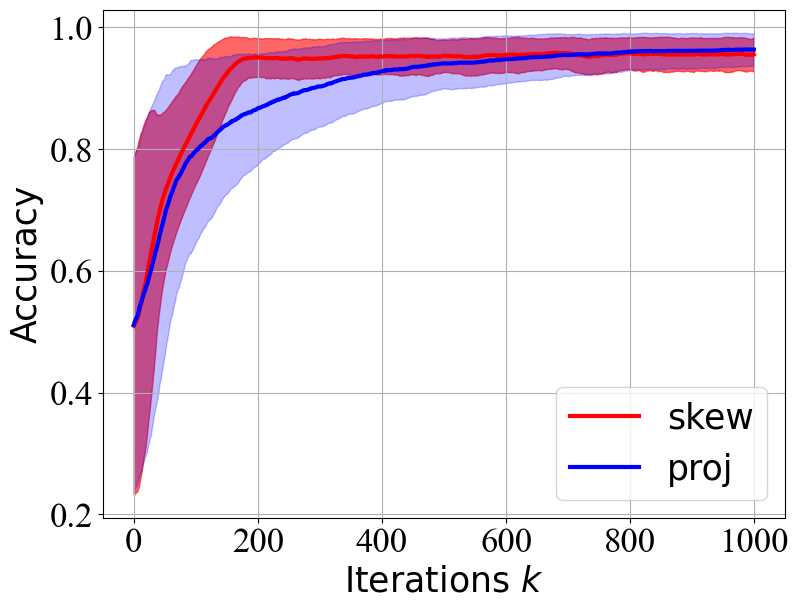}
        \caption{Accuracy over the test set}
    \end{subfigure}%
    \caption{Accuracy over the training set and the test set for the synthetic data. The red part denotes the mean and standard deviation of SRNSGLD and the blue part denotes the mean and standard deviation of PSGLD.}
    \label{logsyth}    
\end{figure}

%%%%%%%%%%%%%%%%

\subsubsection{Real data}

In this section, we consider the constrained Bayesian logistic regression problem on the MAGIC Gamma Telescope dataset \footnote{The Telescope dataset is publicly available from:\\ \href{https://archive.ics.uci.edu/ml/datasets/magic+gamma+telescope}{\texttt{https://archive.ics.uci.edu/ml/datasets/magic+gamma+telescope.}} }and the Titanic dataset \footnote{The Titanic dataset is publicly available from:\\ \href{https://www.kaggle.com/c/titanic}{\texttt{https://www.kaggle.com/c/titanic}.}}. The Telescope dataset contains $n=19020$ samples with dimension $d=10$, describing the registration of high energy gamma particles in a ground-based atmospheric Cherenkov gamma telescope using the imaging technique. The Titanic dataset contains $891$ samples with $10$ features representing information about the passengers. And the goal is to predict whether a passenger survived or not based on these features. However, some of the features are irrelevant to the our goal, such as the ID number. Therefore, after a preprocessing on the raw dataset, we obtained a dataset containing $n=891$ labeled samples with $d=8$ features. For both of the datasets, we initialized the SRNSGLD and PSGLD with the uniform distribution on the centered unit ball in respective dimensions.

For the Telescope dataset, we set stepsize $\eta = 10^{-4}$ and batch size $b=100$ and ran both SRNSGLD and PSGLD $1000$ iterations over the training set, where we chose the skew matrix as $J_{1.5}$ defined in~\eqref{eqn:skew:matrix} with $a=1.5$ for SRNSGLD. Figure~\ref{tele} reports the accuracy of two algorithms over the training set and the test set where the test set counts for $20\%$ of the whole dataset. 
\begin{figure}[htbp]
    \centering
    \begin{subfigure}{0.35\textwidth}
        \centering
        \includegraphics[width=\linewidth]{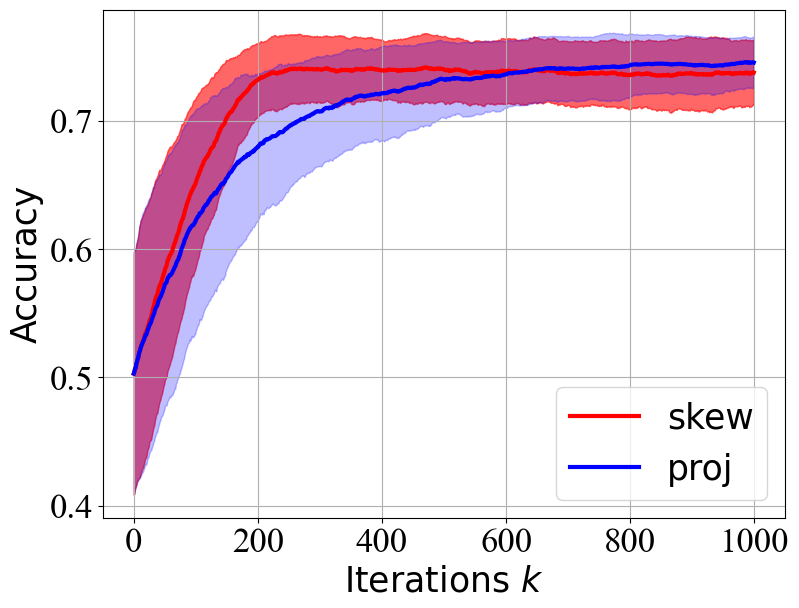}
        \caption{Accuracy over the training set}
    \end{subfigure}%
    \hspace{0.8cm}
    \begin{subfigure}{0.35\textwidth}
        \centering
        \includegraphics[width=\linewidth]{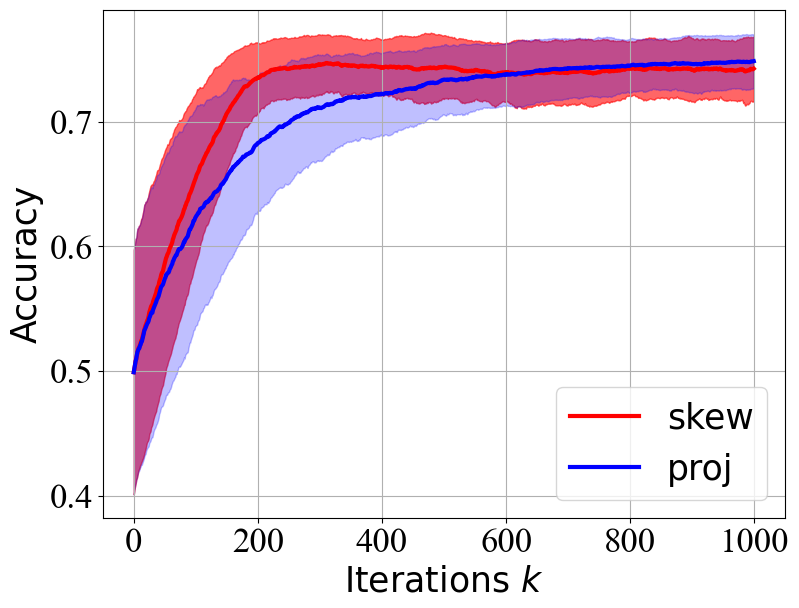}
        \caption{Accuracy over the test set}
    \end{subfigure}%
    \caption{Accuracy over the training set and the test set for the telescope dataset. The red part denotes the mean and standard deviation of SRNSGLD and the blue part denotes the mean and standard deviation of PSGLD.}
    \label{tele}    
\end{figure}

For the Titanic data, we set stepsize $\eta = 10^{-4}$ and batch size $b=50$ and ran both SRNSGLD and PSGLD $1500$ iterations over the training set, where we chose the skewed matrix $J_2$ defined in~\eqref{eqn:skew:matrix} with $a=2$ for SRNSGLD. Figure~\ref{Titanic} reported the accuracy level of two algorithms over the training and test sets where the test set counts for $20\%$ of the whole dataset. 

\begin{figure}[htbp]
    \centering
    \begin{subfigure}{0.35\textwidth}
        \centering
        \includegraphics[width=\linewidth]{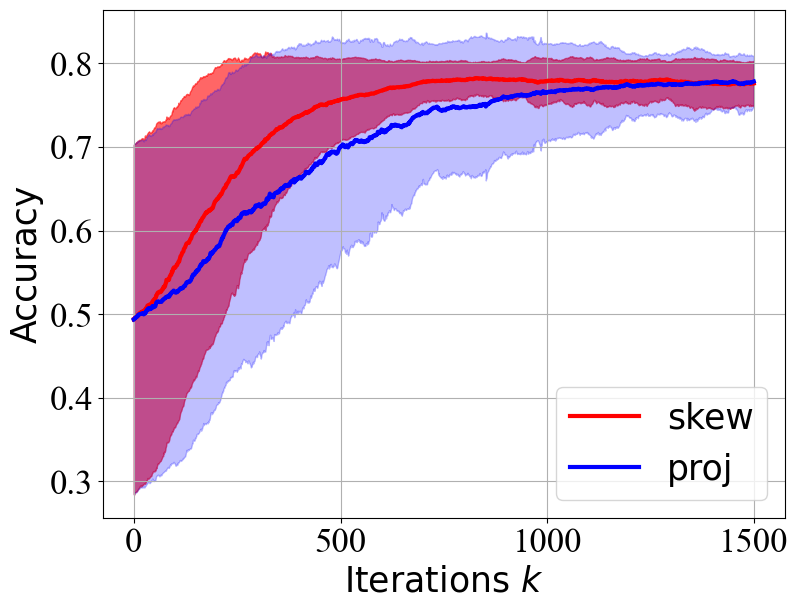}
        \caption{Accuracy over the training set}
    \end{subfigure}%
    \hspace{0.8cm}
    \begin{subfigure}{0.35\textwidth}
        \centering
        \includegraphics[width=\linewidth]{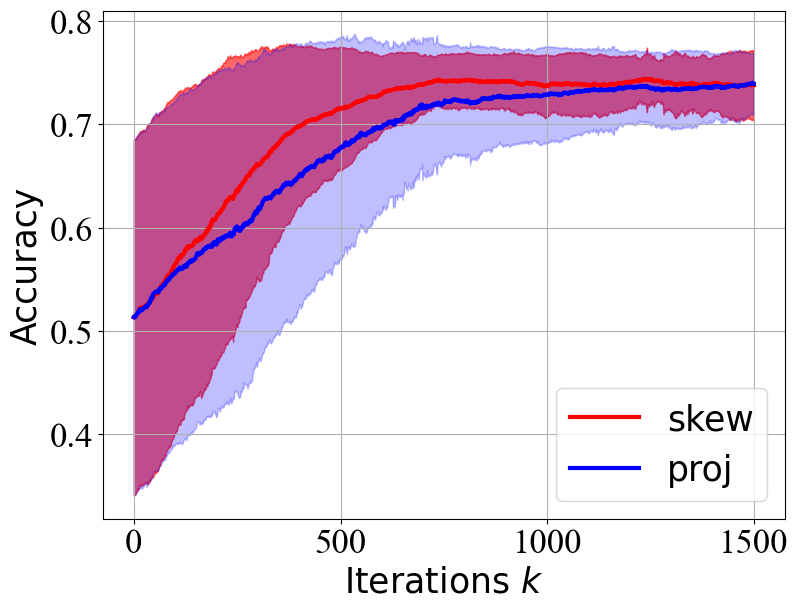}
        \caption{Accuracy over the test set}
    \end{subfigure}%
    \caption{Accuracy over the training set and the test set for the Titanic dataset. 
 The red part denotes the mean and standard deviation of SRNSGLD and the blue part denotes the mean and standard deviation of PSGLD.}
    \label{Titanic}    
\end{figure}

We conclude from Figure~\ref{tele} and Figure~\ref{Titanic} that the highest accuracy SRNSGLD and PSGLD can attained is in the similar level for both Telescope and Titanic datasets. However, The accuracy level of SRNSGLD improved at a faster rate than that of PSGLD. This experiment demonstrates that SRNSGLD is practically efficient.

%%%%%%%%%%%%%%%%%%%%%%%%%%%%%%%%%%
\section{Conclusion}

In this paper, we studied the constrained sampling problem to sample from a target distribution on a constrained domain. 
We proposed and studied SRNLD, a continuous-time SDE with skew-reflected boundary. 
We obtained non-asymptotic convergence rate of SRNLD to the target distribution in both total variation and $1$-Wasserstein distances.
By breaking reversibility, we showed that the convergence rate is better than
the special case of the reversible dynamics.
We also proposed SRNLMC, based on the discretization of SRNLD, 
and obtained non-asymptotic 
discretization error from SRNLD, and provided convergence guarantees
to the target distribution in $1$-Wasserstein distance, that has better performance guarantees than PLMC based on the reversible dynamics 
in the literature. 
Hence, the acceleration is achievable by breaking reversibility in the context
of constrained sampling.
Numerical experiments are provided for both synthetic data
and real data to show efficiency of the proposed algorithms.

%%%%%%%%%%%%%%%%%%%%%%%%%%%%%%%%%%
\section*{Acknowledgements}

We would like to thank Mert G\"{u}rb\"{u}zbalaban, Yuanhan Hu and Yingli Wang. 
Qi Feng is partially supported by the grants NSF DMS-2306769 and DMS-2420029. 
Xiaoyu Wang is supported by the Guangzhou-HKUST(GZ) Joint Funding Program (No.2024A03J0630, No.2025A03J3556), Guangzhou Municipal Key Laboratory of Financial Technology Cutting-Edge Research.
Lingjiong Zhu is partially supported by the grants NSF DMS-2053454 and DMS-2208303. 

%%%%%%%%%%%%%%%%%%%%%%%%%%%%%%%%%%%%%%%%%%%%%%%%%%%%%%%%
\bibliographystyle{alpha}
\bibliography{reflected}
%%%%%%%%%%%%%%%%%%%%%%%%%%%%%%%%%%%%%%%%%%%%%%%%%%%%%%%%

\end{document}